%% file: Double_bandits_arxiv.tex
\documentclass{article}

\usepackage{times,verbatim}
\usepackage{graphicx} 
\usepackage{subcaption} 

\usepackage{color}
\usepackage{algorithm}
\usepackage{algorithmic}

\usepackage{amsmath, amssymb, epsfig, multirow, array, mathtools, bbold,bm,pifont,dsfont,amsthm,authblk}
\usepackage{fullpage}

\newcommand{\cC}{\mathcal{C}}
\newcommand{\cS}{\mathcal{S}}
\newcommand{\cI}{\mathcal{I}}

\newcommand{\sigmahat}{\widehat{\sigma}}
\newcommand{\sigmahatmin}{\widehat{\sigma}_{\min}}
\newcommand{\sigmamin}{\sigma_{\min}}
\newcommand{\cO}{\mathcal{O}}
\newcommand{\R}{\mathbb{R}}
\newcommand{\bbE}{\mathbb{E}}

\newcommand{\bbR}{\mathbb{R}}
\newcommand{\bbP}{\mathbb{P}}

\renewcommand{\it}{i_t}
\newcommand{\jt}{j_t}
\newcommand{\Zt}{Z_t}
\newcommand{\istar}{i_{\star}}
\newcommand{\jstar}{j_{\star}}

\newcommand{\vct}[1]{\bm{#1}}
\newcommand{\mtx}[1]{\bm{#1}}

\newcommand{\p}{\vct{p}}
\newcommand{\q}{\vct{q}}
\newcommand{\vu}{\vct{u}}
\newcommand{\vv}{\vct{v}}

\newcommand{\vp}{\vct{p}}
\newcommand{\va}{\vct{a}}

\newcommand{\ve}{\vct{e}}
\newcommand{\mA}{\mtx{A}}
\newcommand{\mE}{\mtx{E}}
\newcommand{\mR}{\mtx{R}}
\newcommand{\mM}{\mtx{M}}
\newcommand{\mS}{\mtx{S}}

\newcommand{\mP}{\mtx{P}}
\newcommand{\mZ}{\mtx{Z}}
\newcommand{\mL}{\mtx{L}}

\newcommand{\mU}{\mtx{U}}
\newcommand{\mSig}{\mtx{\Sigma}}
\newcommand{\mC}{\mtx{C}}
\newcommand{\mW}{\mtx{W}}

\newcommand{\lV}{\left\Vert}
\newcommand{\rV}{\right\Vert}
\newcommand{\lv}{\left\vert}
\newcommand{\rv}{\right\vert}

\newcommand{\bC}{\widehat{\mC}}
\newcommand{\bW}{\widehat{\mW}}
\newcommand{\bL}{\widehat{\mL}}
\newcommand{\bA}{\widehat{\mA}}

\newcommand{\uZt}{\vu_{\Zt}}

\newcommand{\loo}{\mL_{1,1}}
\newcommand{\ltt}{\mL_{2,2}}
\newcommand{\lot}{\mL_{1,2}}
\newcommand{\loj}{\mL_{1,j}}
\newcommand{\ltj}{\mL_{2,j}}
\newcommand{\lio}{\mL_{i,1}}
\newcommand{\lit}{\mL_{i,2}}

\newcommand{\cmark}{\ding{51}}%
\newcommand{\xmark}{\ding{55}}%

\DeclareMathOperator{\argmin}{arg\,min}
\DeclareMathOperator{\argmax}{arg\,max}
\DeclareMathOperator{\Bern}{Bern}
\DeclareMathOperator{\Var}{Var}
\DeclareMathOperator{\Bin}{Binomial}
\DeclareMathOperator{\Mult}{Mult}
\DeclareMathOperator{\reg}{reg}
\DeclareMathOperator{\regt}{\reg_t}
\newcommand{\beq}{\begin{equation}}
\def\ones{\mathds{1}}
\newcommand{\eeq}{\end{equation}}
\newcommand{\rank}{\operatorname{rank}}

\renewcommand{\algorithmicrequire}{\textbf{Input:}~}
\renewcommand{\algorithmicensure}{\textbf{Output:}~}
\newcommand{\leqa}{\mbox{$\;\stackrel{\mbox{\rm (a)}}{\leq}\;$}}
\newcommand{\leqb}{\mbox{$\;\stackrel{\mbox{\rm (b)}}{\leq}\;$}}
\newcommand{\leqc}{\mbox{$\;\stackrel{\mbox{\rm (c)}}{\leq}\;$}}
\newcommand{\leqd}{\mbox{$\;\stackrel{\mbox{\rm (d)}}{\leq}\;$}}

\makeatletter
\@addtoreset{algorithm}{section}
\makeatother

\makeatletter
\newenvironment{game}[1][htb]
  {\renewcommand{\ALG@name}{Model}
   \begin{algorithm}[#1]%
  }{\end{algorithm}}
\makeatother

\newcommand{\defeq}{\mathrel{\mathop:}=}

\newtheorem*{lemma*}{Lemma}
\newtheorem*{theorem*}{Theorem}
\newtheorem*{corollary*}{Corollary}
\newtheorem*{proposition*}{Proposition}
\newtheorem{theorem}{Theorem}[section]
\newtheorem{lemma}[theorem]{Lemma}
\newtheorem{corollary}[theorem]{Corollary}
\newtheorem{proposition}[theorem]{Proposition}

\newtheorem{remark}[subsection]{Remark}

\usepackage{hyperref}

\title{Active Algorithms For Preference Learning Problems with Multiple Populations}

\begin{document} 

\author[1]{Aniruddha Bhargava\thanks{aniruddha@wisc.edu}}
\author[2]{Ravi Sastry Ganti\thanks{gmravi2003@gmail.com.}}
\author[1]{Robert Nowak\thanks{nowak@ece.wisc.edu\\ \hspace{5pt}The first two authors made equal contributions to the paper. \\ \hspace{5pt} Most of the work was done while RSG was at UW-Madison.}}
\affil[1]{University of Wisconsin, Madison WI}
\affil[2]{Walmart Labs, San Bruno, CA}
\maketitle
\begin{abstract} 
In this paper we model the problem of learning preferences of a population as an active learning problem. We propose an algorithm can adaptively choose pairs of items to show to users coming from a heterogeneous population, and use the obtained reward to decide which pair of items to show next.  We provide computationally efficient algorithms with provable sample complexity guarantees for this problem in both the noiseless and noisy cases. In the process of establishing sample complexity guarantees for our algorithms, we establish new results using a Nystr{\"o}m-like method which can be of independent interest. We supplement our theoretical results with experimental comparisons. 
\end{abstract} 

\section{Introduction}
\label{sec:intro}
\input{Intro_Rob.tex}

\section{Related Work}
\label{sec:relatedwork}
\input{RelatedWork.tex}

\section{Models for preference learning problem}
\label{sec:diff_games}
\input{diff_games.tex}

\section{Low rank structure in our models}
\label{sec:low_rank}
\input{low_rank.tex}

\section{Algorithms for deterministic model}
\label{sec:alg_det}
\input{alg_det.tex}

\section{Algorithms for stochastic model}
\label{sec:alg_stoc}
\input{alg_stoc.tex}

\section{Experiments}
\label{sec:expts}
\input{expts_arxiv.tex}

\section{Conclusions}
\vspace{-10pt}
In this paper we introduced models for the preference learning problem and showed how to design active learning algorithms via reduction to completion of SPSD matrices. An interesting future work is to design algorithms and models when more than two choices are made each time. One can expect that techniques built in this paper adapted to tensor structure will be useful in such problems.
\bibliographystyle{unsrt}
\newpage
\appendix
\input{appendix_arxiv.tex}
{\small{\bibliography{Double_bandits_arxiv}}}
\end{document}

%% file: Intro_Rob.tex
In this work, we are interested in designing active learning algorithms for preference learning problems where there are multiple sub-populations of users, and users from the same sub-population have similar preferences over a large set of items. We propose models where the algorithm chooses a pair of items for each incoming user and obtains a scalar reward which is possibly stochastic. The reward is large when either of the chosen pair of items is liked by the user. Our goal is to find an optimal pair of items by actively eliciting responses from users on different pairs of items. 

As a concrete example, suppose an advertising engine can show just two advertisements to each incoming user.  Different ads may appeal more or less to different sub-populations, but the right combination of two ads optimizes the chance that a random user will like one or the other.  This paper focuses on selecting a pair of items, rather than larger subsets, for the sake of simplicity.  We will discuss generalizations and extensions later in the paper.


Preference modeling problems such as the one described above and more are ubiquitous and show up in internet advertising~\cite{babaioff2009characterizing}, econometrics~\cite{hensher1998combining}, and marketing~\cite{louviere1983design}. An informal description of the preference learning problem that we consider in this paper is as follows: We are given a very large population of users and a set of $K$ items. In most of the applications that we are interested in $K$ is large. Suppose each user belongs to one of $r$ unknown sub-populations of the original population. We are interested in learning from user data which pair of items, among the set of $K$ items are maximally preferred.  The contributions of this paper can be summarized as follows
\begin{enumerate}
\item We introduce two models for preference the modeling problem, when there are $r$ sub-populations and $K$ items. In each round we choose a pair of items and receive a reward which is a function of the pair chosen. This reward is large if either of the items in the chosen pair is ``good''. For both the models we are interested in designing algorithms that discover an $(\epsilon,\delta)$ best pair of items using as few trials as possible i.e. algorithms which can output, with probability at least $1-\delta$, a pair of items that is $\epsilon$ close to the best pair of items in terms of the expected reward of the pair. The difference between the models is  whether the reward is stochastic or deterministic.
\item \textit{The core idea behind our proposed algorithms is to reduce the problem of finding a near-optimal pair of items to the problem of finding a near-smallest element of a certain low-rank, symmetric positive semi-definite matrix (SPSD)}. The rank of the matrix is equal to the number of sub-populations. While one could, in principle, use low-rank matrix completion (LRMC) techniques as in~\cite{candes2009exact}, our proposed algorithms explicitly exploit the SPSD structure in the problem. 
\item The algorithms that we propose for both the models are based on iteratively finding linearly independent columns of an SPSD matrix. The key challenge in designing this column selection step is that querying elements of the SPSD matrix is either very expensive, or noisy, and hence we need to be selective about what elements of the matrix to query. 
\item For all of the algorithms we establish sample complexity bounds of finding an $(\epsilon,\delta)$ optimal pair of items. Establishing such sample complexity bounds leads to interesting problems in matrix approximation in the max-norm. The contributions we make here could be of independent interest in the low-rank matrix approximation literature.
\end{enumerate}
To the best of our knowledge our formulation of preference learning problems as bandit problems with a large number of items and identifying the matrix structure in the items and deriving bounds in terms of rank of the underlying matrix is new and is what makes this work novel compared to other relevant literature.

\textbf{Notation.}
 $\Delta_r$ represents the $r$ dimensional probability simplex. Matrices and vectors are represented in bold font. For a matrix $\mL$, unless otherwise stated, the notation $\mL_{i,j}$ represents $(i,j)$ element of $\mL$, and $\mL_{i:j,k:l}$ is the submatrix consisting of rows $i,i+1,\ldots,j$ and columns $k,k+1,\ldots,l$. The matrix $\lV \cdot \rV_1$ and $\lV \cdot \rV_2$ norms are always operator norms. The matrix $\lV \cdot \rV_{\max}$ is the element wise infinity norm. Finally, let $\ones$ be the all $1$ column vector.

%% file: RelatedWork.tex
Our framework is closely tied to literature on pure exploration problems in multi-armed bandits. In such pure exploration problems one is interested in designing algorithms with low simple regret or designing algorithms with low $(\epsilon,\delta)$ query complexity. Algorithms with small simple regret have been designed in the past ~\cite{audibert2010best,gabillon2011multi,bubeck2012multiple}. ~\cite{even2006action} suggested the Successive Elimination (SE) and Median Elimination (ME) to find near optimal arms with provable sample complexity guarantees. These sample complexity guarantees typically scale linearly with the number of arms. In principal, one could naively reduce our problem to a pure exploration problem where we need to find an $(\epsilon,\delta)$ good arm. However, such naive reductions throw away any dependency information among the arms. The algorithms that we give build on the SE algorithm but crucially exploits the matrix structure in the problem to give much better algorithms than a naive reduction.


Bandit problems where multiple actions are selected have also been considered in the past. ~\cite{kale2010non} consider a setup where on choosing multiple arms the reward obtained is the sum of the rewards of the chosen arms, and the reward of each chosen arm is reveled to the algorithm. Both these works focus on obtaining guarantees on the cumulative regret compared to the best set of arms in hindsight. ~\cite{radlinski2008learning} consider a problem, in the context of information retrieval, where multiple bandit arms are chosen and the reward obtained is the maximum of the rewards corresponding to the chosen arms. Apart from this reward information the algorithm also gets a feedback which is which one of the chosen arms has the highest reward.~\cite{streeter2009online,yue2011linear} also study similar models. A major difference between the above mentioned works and our work is the feedback and reward model and the fact that we are not interested in regret guarantees but rather in finding a good pair of arms as quickly as possible. Furthermore our linear-algebraic approach to the problem is very different from previous approaches which were either based on multiplicative weights~\cite{kale2010non} or online greedy submodular maximization~\cite{streeter2009online,yue2011linear,radlinski2008learning}. ~\cite{simchowitz2016best} also consider similar subset selection problems and provide algorithms to identify the top set of arms. In the Web search literature click models have been proposed to model user behaviour~\cite{guo2009click,craswell2008experimental} and a bandit analysis of such models have also been proposed~\cite{kveton2015cascading}. However, these models assume that all the users come from a single population and tend to use richer information in their formulations (for example information about which exact link was clicked). Interactive collaborative filtering (CF) and bandit approaches to such problems have also been investigated\cite{kawale2015efficient}. Though, the end goal in CF is different from our goal in this paper.

In our paper, we reduce the problem of preference learning to a low-rank SPSD matrix completion problem both in the noiseless and noisy cases. Many other authors have studied the problem of SPSD matrix completion~\cite{bishop2014deterministic}. However, all of these papers consider the passive case, i.e. the entries of the matrix that have been revealed are not under their control. In contrast, we have an active setup, where we can decide which entries in the matrix to reveal. The Nystr{\"o}m algorithm for approximation of low rank SPSD matrices has been well studied both empirically and theoretically. Nystr{\"o}m methods typically choose random columns to approximate the original low-rank matrix~\cite{gittens2013revisiting,drineas2005nystrom}. Adaptive schemes where the columns used for Nystrom approximation are chosen adaptively have also been considered in the literature. To the best of the knowledge these algorithms either need the knowledge of the full matrix~\cite{deshpande2006matrix} or have no provable theoretical guarantees~\cite{kumar2012sampling}. Moreover, to the best of our knowledge all analysis of Nystrom approximation that has appeared in the literature assume that one can get error free values for entries in the matrix which does not apply to the problems that we are interested in.


%% file: diff_games.tex
In this section, we propose two models for the preference modeling problem. In both the proposed models the algorithm chooses a pair of items, and receives a reward which is a function of the the chosen pair and possibly other external random factors. 
For both the models we want to design algorithms that discover, using as few trials as possible, an $(\epsilon,\delta)$ best pair of items. The difference between the two models is that in the first model the reward obtained is a deterministic function of the pair of items chosen, whereas in our second model the reward obtained is random.  
\begin{game}
\caption{Description of our proposed models \label{models}}
\begin{algorithmic}[1]
\WHILE{TRUE}
\STATE In the case of stochastic model, nature chooses $\Zt\sim \Mult(\p)$, but does not reveal it to the algorithm.
\STATE Algorithm chooses a pair of items $(i_t,j_t)$.
\STATE Algorithm receives the reward $y_t$ defined as follows
\begin{align}
y_{t,\text{det}}&=1-\bbE_{Z_t\sim \vp} (1-\vu_{Z_t}(\it))(1-\vu_{Z_t}(\jt)) ~\text{// if model is deterministic}\label{eqn:reward_det}\\
y_{t,\text{stoc}}&=\max\{y_{\it}, y_{\jt}\} ~\text{// if model is stochastic}\label{eqn:reward_stoc}\\
y_{\it}&\sim\Bern(\uZt(\it))\\
y_{\jt}&\sim\Bern(\uZt(\jt)) 
\end{align}
\STATE Algorithm stops if it has found a certifiable $(\epsilon,\delta)$ optimal pair of items.
\ENDWHILE
\end{algorithmic}
\end{game}
In Figure~\eqref{models} we sketch both the deterministic and the stochastic model. Let, $Z_t$ be a multinomial random variable defined by a probability vector $\vp\in \Delta_r$, whose output space is the set $\{1,2,\ldots, r\}$. Let $\uZt$ be a reward vector in $[0,1]^K$ indexed by $Z_t$. On playing the pair of arms $(i_t,j_t)$ in round $t$ the algorithm receives a scalar reward $y_t$. In the deterministic model $y_t$ is deterministic and is given by Equation~\eqref{eqn:reward_det}, whereas in the stochastic model $y_t$ is a random variable that depends on the random variable $Z_t$ as well as additional external randomness. However, a common aspect of both these models is that the expected reward associated with the pair of choices $(i_t,j_t)$ in round $t$ is the same and is equal  to the expression in given in Equation~\eqref{eqn:reward_det}

As a concrete example of the above models consider the example that we discussed in the introduction. Here, we have a random customer that belongs to one of the three possible categories with probability defined by a probability vector $\vp\in \Delta_3$. When using a stochastic model, this advertising company which has no information about which sub-population this customer is from,  makes \$$1$ whenever, this customer clicks on either of the one of the two displayed advertisements, and \$$0$ otherwise. This reward depends on two sources of randomness. (i) The sub-population that the user comes from which is modeled by a multinomial random variable $Z_t$ and (ii) external randomness that is modeled by Bernoulli random variables with parameters $\uZt(\it)$ and $\uZt(\jt)$. In the case of a deterministic model the reward obtained is the same as the expected reward in stochastic model, and therefore not random. 
It is clear from Figure~\eqref{models} that the optimal pair of choices satisfies the equation
\begin{equation}
\label{eqn:obj_models1}  
(\istar,\jstar)=\argmin_{i,j} \bbE_{Z_t\sim \vp} (1-\vu_{Z_t}(i))(1-\vu_{Z_t}(j)).
\end{equation}
Since, we are interested in returning an $(\epsilon,\delta)$ optimal pair of choices it is enough if the pair returned by our algorithm attains an objective function value that is at most $\epsilon$ more than the optimal value of the objective function shown in equation~\eqref{eqn:obj_models1}, with probability at least $1-\delta$.

%% file: low_rank.tex
Let $\vp\in \Delta_r$, and let the reward matrix $\mR\in \bbR^{K\times K}$ be such that its $(i,j)^{\text{th}}$ entry is the probability of obtaining a reward of $1$ when the pair of items $(i,j)$ are pulled. Then from equation~\eqref{eqn:reward_det} and equation~\eqref{eqn:reward_stoc} we know that the reward structure for both the deterministic and stochastic models has the form
\begin{align}
\mR_{i,j}&=1-\bbE_{Z_t\sim \vp} (1-\vu_{Z_j}(i))(1-\vu_{Z_j}(j))\nonumber\\
&=1-\sum_{k=1}^r \vp_k (1-\vu_k(i))(1-\vu_k(j))\\ 
\mR&=\ones\ones^\top-\underbrace{\sum_{k=1}^r \vp_k (\ones-\vu_k)(\ones-\vu_k)^\top}_{\mL}\label{eqn:low_rank}.
\end{align}
As mentioned before our goal is to find a pair of items that are $(\epsilon,\delta)$ optimal. Hence, it is enough to find an entry in the matrix $\mL$ that is $\epsilon$ close to the smallest entry in the matrix $\mL$ with probability at least $1-\delta$. A naive way to solve this problem is to treat this problem as a best-arm identification problem in stochastic multi-armed bandits where there are $\Theta(K^2)$ arms each corresponding to a pair of items.  One could now run a Successive Elimination (SE) algorithm  or a Median Elimination algorithm on these $\Theta(K^2)$ pairs~\cite{even2006action} to find an $(\epsilon,\delta)$ optimal pair. The sample complexity of the SE or ME algorithms on these $\Theta(K^2)$ pairs would be roughly $\tilde{\cO}(\frac{K^2}{\epsilon^2})$~\footnote{The $\tilde{\cO}$ notation hides logarithmic dependence on $\frac{1}{\delta}, K, \frac{1}{\delta}$}. In typical applications that we are interested in, $K$ can be very large, and therefore the sample complexity of such naive algorithms can be very large. However, these simple reductions throw away information between different pairs of items and hence are sub-optimal. A natural question to ask is can we design algorithms that can efficiently exploit the matrix structure in our problem? It turns out that in our problem we can get away with sample complexity far smaller than $K^2$. In order to do this we exploit the structural properties of matrix $\mL$. From equation~\eqref{eqn:low_rank} it is clear that the matrix $\mL$ can be written as a sum of $r$ rank-1 matrices. Hence $\rank(\mL)\leq r$. Furthermore, since these rank-1 matrices are all positive semi-definite and $\mL$ is a convex combination of such, we can conclude that $\mL\succeq 0$. We exploit these properties in our algorithm design. We have proved the following proposition:
\begin{proposition}
The matrix $\mL$ shown in equation~\eqref{eqn:low_rank} satisfies the following two properties: (i) $\rank(\mL)\leq r$ (ii) $\mL\succeq 0$.
\end{proposition}

%% file: alg_det.tex
Our approach to finding an $(\epsilon,\delta)$ pair in the deterministic model considered in this paper is via matrix completion. Choosing the pair of items $(i,j)$ reveals the $(i,j)^{th}$ entry of the matrix $\mR$. As can be seen from equation~\eqref{eqn:low_rank} the matrix $\mR$ has rank at most $r+1$. One can use standard low rank matrix completion (LRMC) ideas which relies on nuclear norm minimization techniques~\cite{candes2009exact}. A typical result from the LRMC literature says that we need to see $\cO(Kr\mu\log^2(K))$ random entries in the matrix $\mR$, where $\mu$ is the upper bound on the coherence of the row and column space of matrix $\mR$, for exact recovery of the matrix $\mR$. We provide a simple algorithm that recovers the matrix $\mR$ after querying $\Theta(Kr)$ entries in the matrix $\mR$. Our algorithm called PLANS~\footnote{Preference Learning via Adaptive Nystrom Sampling} is shown in Figure~\eqref{alg:PLANS} and works with the matrix $\mL$~\footnote{Pulling the pair (i,j) gets us $\mR_{i,j}$, which can be used to populate $\mL_{i,j}=1-\mR_{i,j}$}.

PLANS is shown in Figure~\eqref{alg:PLANS}. It is an iterative algorithm that finds out what columns of the matrix are independent. PLANS maintains a set of indices (denoted as $\cC$ in the pseudo-code) corresponding to independent columns of matrix $\mL$. Initially $\cC=\{1\}$. PLANS then makes a single pass over the columns in $\mL$ and checks if the current column is independent of the columns in $\mC$. This check is done in line $5$ of Figure~\eqref{alg:PLANS} and most importantly requires \textit{only the principal sub-matrix}, of $\mL$,  indexed by the set $\cC\cup \{c\}$.  If the column passes this test then all the elements in this column $i$
 whose values have not been queried in the past are queried and the matrix $\hat{\mL}$ is updated with these values. The test in line $5$ is the column selection step of the PLANS algorithm and is justified by Proposition~\eqref{prop:degenerate}. Finally, once $r$ independent columns have been chosen, we impute the matrix by using Nystrom extension. Nystrom based methods have been proposed in the past to handle large scale kernel matrices in the kernel based learning literature~\cite{drineas2005nystrom,kumar2012sampling}. The major difference between this work and ours is that the column selection procedure in our algorithms is deterministic, whereas in Nystrom methods columns are chosen at random. The following proposition simply follows from the fact that any principal submatrix of an SPSD matrix is also SPSD and hence admits an eigen-decomposition.
\begin{algorithm}
\caption{Preference Learning via Adaptive Nystrom Sampling (PLANS)}\label{alg:PLANS}
\algorithmicrequire{A deterministic oracle that takes a pair of indices $(i, j)$ and outputs $\mL_{i,j}$.}

\algorithmicensure{$\hat{\mL}$}
\begin{algorithmic}[1]
\STATE Choose the pairs $(j,1)$ for $j=1,2,\ldots,K$ and set $\hat{\mL}_{j,1}=\mL_{j,1}$. Also set $\hat{\mL}_{1,j}=\mL_{j,1}$
\STATE $\cC = \{1\}$ \COMMENT{Set of independent columns discovered till now}
\FOR{($c=2;~c\gets c+1;~c\leq K$)}
\STATE Query the oracle for $(c,c)$ and set $\hat{\mL}_{c,c}\leftarrow \mL_{c,c}$
\IF {$\sigma_{\min}\left(\hat{\mL}_{\cC\cup \{c\},\cC\cup \{c\}} \right)> 0$}
\STATE $\cC \gets \cC \cup \{c\}$
\STATE Query $\cO$ for the pairs $(\cdot,c)$  and set $\hat{\mL}(\cdot,c)\gets \mL(\cdot,c)$ and by symmetry $\hat{\mL}(c,\cdot)\gets \mL(\cdot,c)$.
\ENDIF
\IF{($|\cC|=r$)}
\STATE break
\ENDIF
\ENDFOR
\STATE  Suppose the submatrix of $\mL$ corresponding to columns in $\cC$  is $\mC$ and the principal submatrix in $\mL$ corresponding to indices in $\cC$ is $\mW$. Then, construct the Nystrom extension  $\bL=\mC\mW^{-1}\mC^\top$.
\end{algorithmic}
\end{algorithm}
\begin{proposition}
\label{prop:degenerate}
Let $\mL$ be any SPSD matrix of size $K$. Given a subset $\cC\subset \{1,2,\ldots,K\}$, the columns of the matrix $\mL$ indexed by the set $\cC$ are independent iff the principal submatrix $\mL_{\cC,\cC}$ is non-degenerate, equivalently iff, $\lambda_{\min}(\mL_{\cC,\cC})>0$. 
\end{proposition}
It is not hard to verify the following theorem. The proof has been relegated to the appendix.
\begin{theorem}
\label{thm:noiseless}
If $\mL \in \R^{K \times K}$ is an SPSD matrix of rank $r$, then the matrix $\hat{\mL}$ output by the PLANS algorithm ~\eqref{alg:PLANS} satisfies $\hat{\mL}=\mL$. Moreover, the number of oracle calls made by PLANS is at most  $K(r+1)$.
The sampling algorithm~\eqref{alg:PLANS} requires: $K + (K-1) + (K-2) + \ldots + (K - (r-1))+(K-r) \leq (r+1)K$ samples from the matrix $\mL$.
\end{theorem}
The following corollary follows immediately.
\begin{corollary}
Using algorithm~\eqref{alg:PLANS} we can output a $(0,0)$ optimal pair of items for the deterministic model by obtaining rewards of at the most $K(r+1)$ pairs of items adaptively.
\end{corollary}
Note that the sample complexity of the PLANS algorithm is slightly better than typical sample complexity results for LRMC. We managed to avoid factors logarithmic in dimension and rank, as well as incoherence factors that are typically found in LRMC results~\cite{candes2009exact}. Also, our algorithm is purely deterministic, whereas LRMC uses randomly drawn samples from a matrix. In fact, this careful, deterministic choice of entries of the matrix is what helps us do better than LRMC. 

Moreover, PLANS algorithm is optimal in a min-max sense. This is because any SPSD matrix of size $K$ and rank $r$ is characterized via its eigen decomposition by $Kr$ degrees of freedom. Hence, any algorithm for completion of an SPSD matrix would need to see at least $Kr$ entries. As shown in theorem~\eqref{thm:noiseless} the PLANS algorithm makes at most $K(r+1)$ queries and hence PLANS is min-max optimal.

The PLANS algorithm needs the rank $r$ as an input. However, the PLANS algorithm can be made to work even if $r$ is unknown by simply removing the condition on line 9 in the PLANS algorithm. Even in this case the sample complexity guarantees in Theorem~\eqref{thm:noiseless} hold. Finally, if the matrix is not exactly rank $r$ but can be approximated by a matrix of rank $r$, then PLANS can be made more robust by modifying line 5 to $\sigma_{\min}\left(\hat{\mL}_{\cC\cup \{c\},\cC\cup \{c\}} \right)\geq \sigma_{\text{thresh}}$, where $\sigma_{\text{thresh}}$ depends on $r$.

%% file: alg_stoc.tex
For the stochastic model considered in this paper we shall propose an algorithm, called R-PLANS~\footnote{R-PLANS stands for Robust Preference Learning via Adaptive Nystr{\"o}m Sampling}, which is a robust version of PLANS. Like PLANS, the robust version discovers a set of independent columns iteratively and then uses the Nystr{\"o}m extension to impute the matrix. Figure~\eqref{alg:R-PLANS} provides a pseudo-code of the R-PLANS algorithm.

R-PLANS like the MISA algorithm repeatedly performs column selection steps to select a column of the matrix $\mL$ that is linearly independent of the previously selected columns, and then uses these selected columns to impute the matrix via a Nystrom extension. In the case of deterministic models, due to the presence of a deterministic oracle, the column selection step is pretty straight-forward and requires calculating the smallest singular-value of certain principal sub-matrices. In contrast, for stochastic models the availability of a much weaker stochastic oracle makes the column selection step much harder. We resort to the successive elimination algorithm where principal sub-matrices are repeatedly sampled to estimate the smallest singular-values for those matrices. The principal sub-matrix that has the largest smallest singular-value determines which column is selected in the column selection step. 

Given a set $\cC$, define $\mC$ to be a $K\times r$ matrix corresponding to the columns of $\mL$ indexed by $\cC$ and define $\mW$ to be the $r\times r$ principal submatrix of $\mL$ corresponding to indices in $\cC$. R-PLANS constructs estimators $\bC,\bW$ of $\mC,\mW$ respectively by repeatedly sampling independent entries of $\mC,\mW$ for each index and averaging these entries. The sampling is such that each entry of the matrix $\mC$ is sampled at least $m_1$ times and each entry of the matrix $\mW$ is sampled at least $m_2$ times, where
\begin{equation}
\label{eqn:m}
m_1 = 100C_1(W,C)\log(2Kr/\delta)\max\left(\frac{r^{5/2}}{\epsilon},\frac{r^2}{\epsilon^2}\right), m_2=200 C_2(W,C)\log(2r/\delta)\max\left(\frac{r^3}{\epsilon},\frac{r^5}{\epsilon^2}\right)
\end{equation}
and $C_1, C_2$ are problem dependent constants defined as
\begin{align}
\label{eqn:C_1_and_C_2}
C_1(\mW,\mC)&=\max\bigl(\lV \mW^{-1}\mC^\top\rV_{\max},\lV \mW^{-1}\mC^\top\rV_{\max}^2,\lV \mW^{-1}\rV_{\max},\lV\mC\mW^{-1}\rV_1^2,\\
	&\hspace{120pt}\lV\mW^{-1}\rV_2\lV\mW^{-1}\rV_{\max}\bigr)\nonumber\\
C_2(\mW,\mC)&=\max\left(\lV \mW^{-1}\rV_2^2 \lV\mW^{-1}\rV_{\max}^2,\lV \mW^{-1}\rV_2 \lV\mW^{-1}\rV_{\max},\lV\mW^{-1}\rV_2,\lV\mW^{-1}\rV_2^2\right)
\end{align}

 R-PLANS then returns the Nystr{\"o}m extension constructed using matrices $\bC,\bW$. 
\begin{algorithm}
\caption{Robust Preference Learning via Adaptive Nystrom Sampling (R-PLANS)}\label{alg:R-PLANS}
\begin{algorithmic}[1]
\REQUIRE $\epsilon>0, \delta>0$ and a stochastic oracle $\cO$ that when queried with indices $(i,j)$ outputs a Bernoulli random variable $\Bern(L_{i,j})$
\ENSURE A SPSD matrix $\hat{\mL}$, which is an approximation to the unknown matrix $\mL$, such that with probability at least $1 - \delta$, all the elements of $\hat{\mL}$ are within $\epsilon$ of the elements of $\mL$. 
\STATE $\mathcal{C} \leftarrow \{1\}$.
\STATE $\mathcal{I} \leftarrow \{2,3,\ldots, K\}$.
\FOR{($t=2;t\gets t+1;t\leq r $)} 
	\STATE Define, $\tilde{\mathcal{C}_i} =\mathcal{C} \bigcup \{i\}, \forall i \in \mathcal{I}$.
	\STATE Run Algorithm~\ref{alg:SEmineig} on matrices $\mL_{\tilde{\cC_i},\tilde{\cC_i}}$,  $i\in \cI$, with given $\delta\gets \frac{\delta}{2r}$ to get $i^{\star}_t$.
	\STATE $\mathcal{C} \leftarrow \mathcal{C} \bigcup \{i^{\star}_t\}; \mathcal{I} \leftarrow \mathcal{I} \setminus \{i^{\star}_t\}$.
\ENDFOR
\STATE Obtain estimators $\bC, \bW$ of $\mC,\mW$ by repeatedly sampling and averaging entries. Calculate the Nystrom extension $\bL=\bC\bW^{-1}\bC^\top$  (see lines 221 - 227).
\end{algorithmic} 
\end{algorithm}
\begin{algorithm}
\caption{Successive elimination on principal submatrices}\label{alg:SEmineig}
\begin{algorithmic}[1]
\REQUIRE   Square matrices $\mA_1, \ldots ,\mA_m$ of size $p\times p$, which share the same $p-1\times p-1$ left principal submatrix; a failure probability $\delta>0$; and a stochastic oracle $\cO$
\ENSURE  An index  
\STATE Set $t=1$, and $\cS=\{1,2,\ldots,m\}$.
\STATE Sample each entry of the input matrices once.
\WHILE{$|\cS|>1$}
\STATE Set $\delta_{t}=\frac{6\delta}{\pi^2mt^2}$
\STATE Let $\sigmahat^{\max}=\underset{k\in \cS}{\max}~ \sigma_{\min}(\bA_k)$ and let $k_{\star}$ be the index that attains argmax.
\STATE For each $k\in \cS$, define $\alpha_{t,k}=\frac{2\log(2p/\delta_{t})}{3~\underset{i,j}{\min}~ n_{i,j}(\bA_k)}+\sqrt{\frac{\log(2p/\delta_{t})}{2}\sum_{i,j}\frac{1}{n_{i,j}(\bA_{k})}}$
 \STATE For each index $k\in \cS$, if $\sigmahat^{\max}-\sigmahatmin(\bA_k)\geq \alpha_{t,k^{\star}}+\alpha_{t,k}$ then do $\cS\leftarrow \cS\setminus \{k\}$.
 \STATE $t\gets t+1$
 \STATE Sample each entry of the matrices indexed by the indices in $\cS$ once.
\ENDWHILE
\STATE Output $k$, where $k\in \cS$.
\end{algorithmic} 
\end{algorithm}

\subsection{Sample complexity of the R-PLANS algorithm}
As can be seen from the R-PLANS algorithm, samples are consumed both in the successive elimination steps as well as during the construction of the Nystr{\"o}m extension. We shall now analyze both these steps next.

\textbf{Sample complexity analysis of successive elimination.}
Before we provide a sample complexity of the analysis, we need a bound on the spectral norm of random matrices with $0$ mean where each element is sampled possibly different number of times. This bound plays a key role in correctness of the successive elimination algorithm. The proof of this bound follows from matrix Bernstein inequality. We relegate the proof to the appendix due to lack of space.
\begin{lemma}
\label{lem:bernstein_diff_n}
Let $\hat{\mP}$ be a $p\times p$ random matrix that is constructed as follows. For each index $(i,j)$ independent of other indices, set $\hat{\mP}_{i,j}=\frac{H_{i,j}}{n_{i,j}}$, where $H_{i,j}$ is a random variable drawn from the distribution $\Bin(n_{i,j},p_{i,j})$. Then, $||\hat{\mP}-\mP||_2 \leq \frac{2\log(2p/\delta)}{3~\underset{i,j}{\min}~ n_{i,j}}+\sqrt{\frac{\log(2p/\delta)}{2}\sum_{i,j}\frac{1}{n_{i,j}}}$.
Furthermore, if we denote by $\Delta$ the R.H.S. in the above bound, then $|\sigmamin(\hat{\mP})-\sigmamin(\mP)|\leq \Delta$. 
\end{lemma}

\begin{lemma}
\label{lem:se}
The successive elimination algorithm shown in Figure~\eqref{alg:SEmineig} on $m$ square matrices of size $\mA_1,\ldots,\mA_m$ each of size $p\times p$ outputs an index $\istar$ such that, with probability at least $1-\delta$, the matrix $\mA_{\istar}$ has the largest smallest singular value among all the input matrices. Let, $\Delta_{k,p}\defeq \max_{j=1,\ldots,m}\sigma_{\min}(\mA_j)-\sigma_{\min}(\mA_k)$. Then number of queries to the stochastic oracle are
\begin{equation}
 \sum_{k=2}^m O\left(p^3\log(2p\pi^2m^2/3\Delta_{k,p}^2\delta)/\Delta_{k,p}^2\right)+O\left(p^4\max_k \log(2p\pi^2m^2/3\Delta_{k,p}^2\delta)/\Delta_{k,p}^2\right)
\end{equation}
\end{lemma}
\textbf{Sample complexity analysis of Nystrom extension.} The following theorem tells us how many calls to a stochastic oracle are needed in order to guarantee that the Nystrom extension obtained by using matrices $\bC,\bW$ is accurate with high probability. The proof has been relegated to the appendix.
\begin{theorem}
\label{thm:nystrom_maxnorm_approx}
Consider the matrix $\bC\bW^{-1}\bC^{\top}$ which is the Nystrom extension constructed in step 10 of the R-PLANS algorithm. Given any $\delta\in (0,1)$, with probability at least $1-\delta$, $\lV \mC \mW^{-1} \mC^\top - \bC \bW^{-1} \bC^\top \rV_{\max}\leq \epsilon$
after making a total of $Krm_1+r^2 m_2$ number of oracle calls to a stochastic oracle, where $m_1, m_2$ are given in equations~\eqref{eqn:m}.
\end{theorem}
The following corollary follows directly from theorem~\eqref{thm:nystrom_maxnorm_approx}, and lemma~\eqref{lem:se}.
\begin{corollary} 
The R-PLANS algorithm outputs an $(\epsilon,\delta)$ good arm after making at most  $$\max\left(Krm_1+r^2m_2, \sum_{p=1}^r\sum_{k=2}^{K-r} \tilde{O}\left(\frac{p^3}{\Delta_{k,p}^2}\right)+\tilde{O}\left(p^4\max_k (\frac{1}{\Delta_{k,p}^2})\right)\right)$$ number of calls to a stochastic oracle, where $\tilde{O}$ hides factors that are logarithmic in $K,r,\frac{1}{\delta},1/\Delta_{k,p}$.
\end{corollary}

%% file: expts_arxiv.tex
\begin{figure}[H]
\centering
\begin{subfigure}{0.33\textwidth}
\centering
\includegraphics[width=\linewidth]{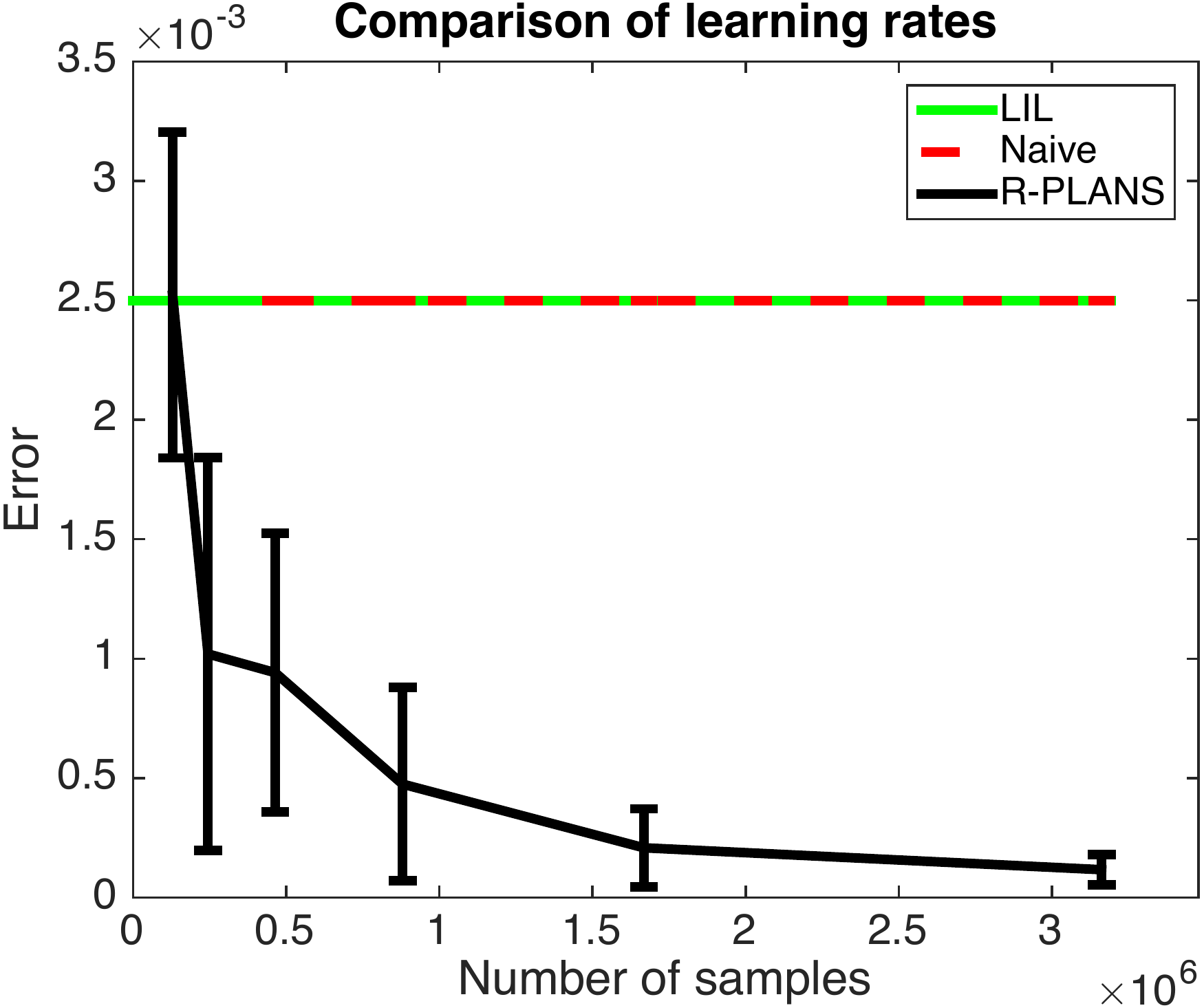}
\caption{ML-100K; $K=800, r=2$\label{fig:real_b}}
\end{subfigure}%
\begin{subfigure}{0.33\textwidth}
\centering
\includegraphics[width=\linewidth]{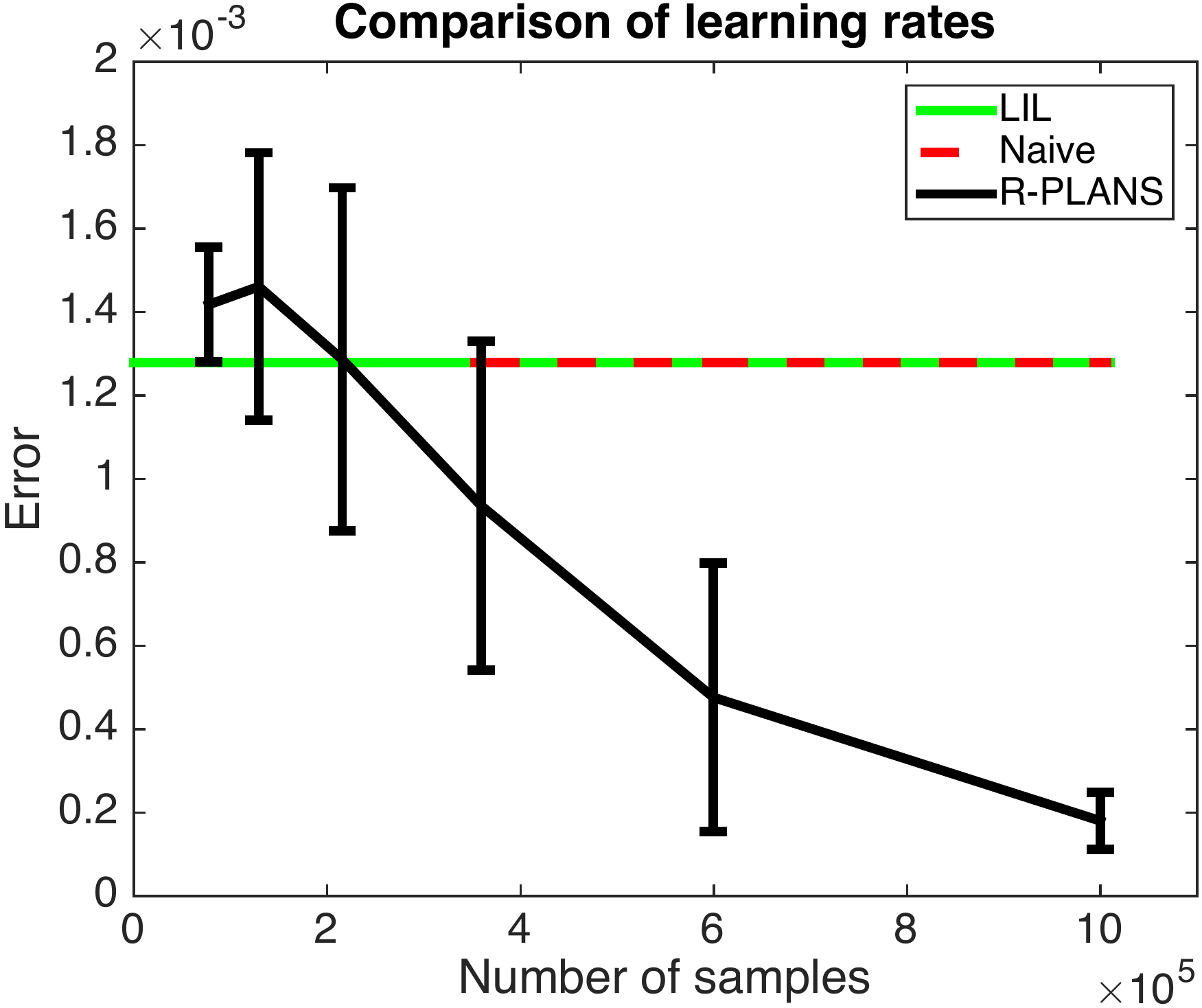}
\caption{ML-100K; $K=800, r=4$\label{fig:real_c}}
\end{subfigure}%
\begin{subfigure}{0.33\textwidth}
\centering
\includegraphics[width=\linewidth]{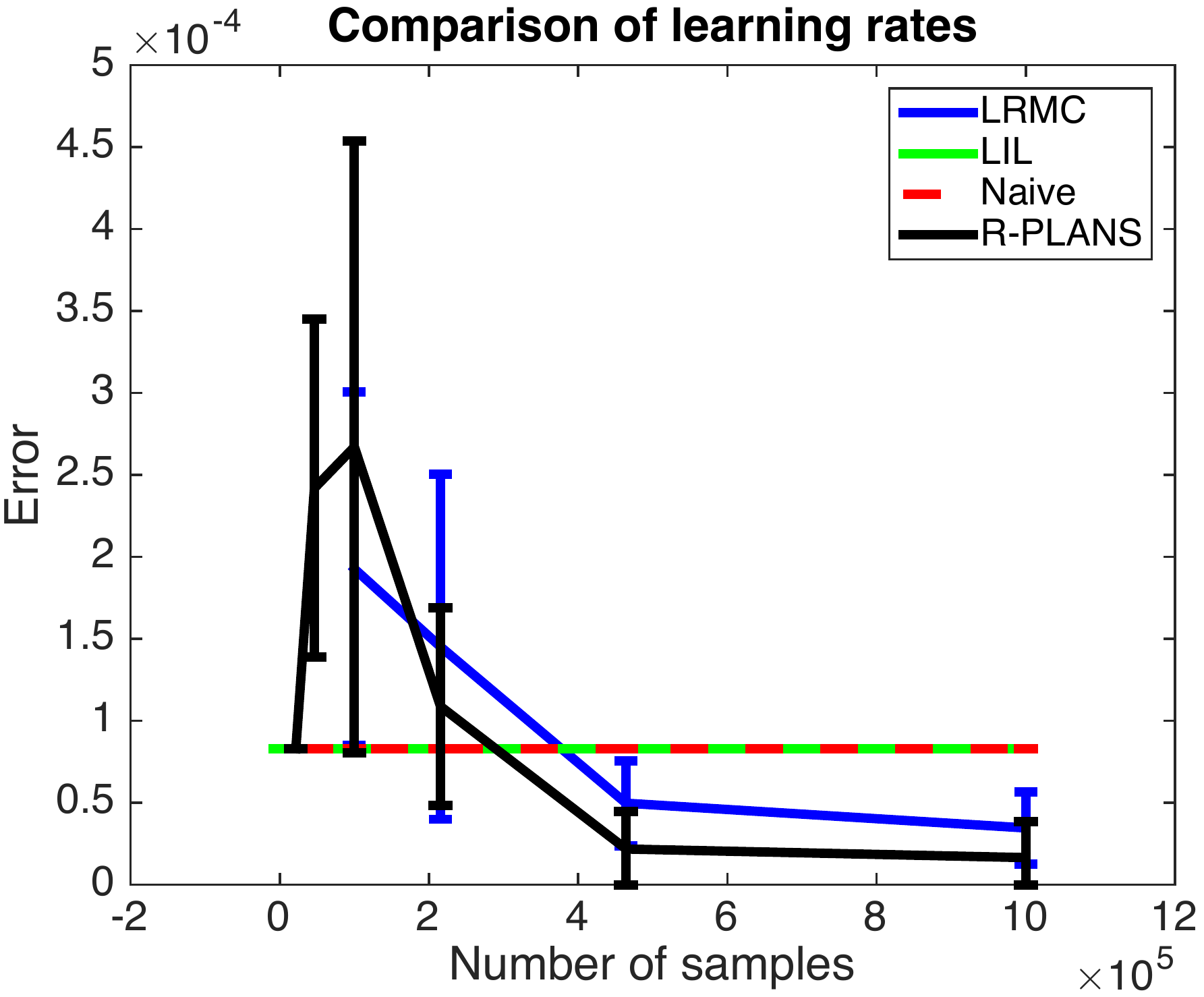}
\caption{ML-1M; $K=200, r=2$\label{fig:real_d}}
\end{subfigure}
\caption{\label{expts}\footnotesize{Error of various algorithms with increasing budget. The error is defined as $\mL_{\hat{i},\hat{j}}-\mL_{i_\star,j_\star}$ where $(\hat{i}, \hat{j})$ is a pair of optimal choices as estimated by each algorithm. Note that the Naive and LIL' UCB have similar performances and do not improve with budget and both are outperformed by R-PLANS. This is because both Naive and LIL' UCB have few samples that they can use on each pair of movies. All experiments were repeated $10$ times.}}
\end{figure}
\vspace{-10pt}
In the experiments demonstrated in this section we provide the different algorithms with increasing budget and measure the error of each algorithm in finding a pair of choices.  The algorithms that we use for comparison are Naive, LiL'UCB~\cite{jamieson2013lil} and LRMC using OptSpace~\cite{keshavan2009matrix}. The naive algorithm uniformly distributes equally the given budget to all $K(K+1)/2$ pairs of choices. LIL applies the LiL' UCB algorithm treating each pair of choices as an arm in a stochastic multi-armed bandit game. 



In Figure~\eqref{expts} we demonstrate results on two Movie Lens datasets~\cite{harper2015movielens}, namely ML-100K, ML-1M. These datasets consist of incomplete user-movie ratings. The task here is to recommend a pair of movies by using this user rating data. The datasets used in our experiments are generated by using the raw dataset and performing LRMC to get missing ratings, followed by thresholding the ratings to obtain binary values. The users were clustered as per their gender in Figures~\eqref{fig:real_b},~\eqref{fig:real_c} and as per their occupation in Figure~\eqref{fig:real_c}. We also sub-sampled the movies in our experiments. As can be seen from the figures the Naive and LIL'UCB algorithms have similar performance on all the datasets. On the ML-100K datasets LIL'UCB quickly finds a good pair of movies but fails to improve with an increase in the budget. To see why, observe that there are about $32\times 10^4$ pairs of movies. The maximum budget here is on the order of $10^6$. Therefore, Naive samples each of those pairs on an average at most four times. Since most entries in the matrix are of the order of $10^{-4}$, Naive algorithm mostly sees 0s when sampling. The same thing happens with the LIL'UCB algorithm too; very few samples are available for each pair to improve its confidence bounds. This explains why the Naive and LIL' UCB algorithm have such poor and similar performances. In contrast, R-PLANS focusses most of the budget on a few select pairs and infers the value of other pairs via Nystrom extensio. This is why, in our experiments we see that R-PLANS finds good pair of movies quickly and finds even better pairs with increasing budget, outperforming all the other algorithms. R-PLANS is also better than LRMC, because we specifically exploit the SPSD structure in our matrix $\mL$, which enables us to do better. We would like to mention  that on ML-100K dataset the performance of LRMC was much inferior. This result and results on synthetic datasets can be found in the appendix 

%% file: appendix_arxiv.tex
\section{Preliminaries}
We shall repeat a proposition that was stated in the main paper for the sake of completeness.
\begin{proposition}
\label{prop:degenerate1}
Let $\mL$ be any SPSD matrix of size $K$. Given a subset $\cC\subset \{1,2,\ldots,K\}$, the columns of the matrix $\mL$ indexed by the set $\cC$ are independent iff the principal submatrix $\mL_{\cC,\cC}$ is non-degenerate, equivalently iff, $\lambda_{\min}(\mL_{\cC,\cC})>0$. 
\end{proposition}

We would also need the classical matrix Bernstein inequality, which we borrow from the work of Joel Tropp~\cite{tropp2015introduction}.
\begin{theorem}
\label{thm:matBern}
Let $\mS_1,\ldots,\mS_n$ be independent, centered random matrices with dimension $d_1\times d_2$ and assume that each one is uniformly bounded
\begin{equation*}
\bbE \mS_k=0, \lV \mS_k\rV \leq L \text{for each } k=1,\ldots, n.
\end{equation*}
Introduce the sum $\mZ=\sum_{k=1}^n \mS_k$, and let $\nu(\mZ)$ denote the matrix variance statistic of the sum:
\begin{align}
\nu(\mZ)&=\max \left\{\lV \bbE \mZ\mZ^\top \rV, \lV \bbE \mZ^\top\mZ \rV\right\}\\
&=\max \left\{\lV \sum_{k=1}^n \bbE \mS_k\mS_k^\top \rV, \lV \sum_{k=1}^n \bbE \mS_k^\top\mS_k \rV\right\}
\end{align}
Then,
\begin{equation*}
\bbP(\lV\mZ\rV\geq t)\leq (d_1+d_2)\exp\left(-\frac{t^2/2}{\nu(\mZ)+\frac{Lt}{3}}\right)
\end{equation*}
\end{theorem}
\section{Sample complexity of PLANS algorithm: Proof of Theorem 5.2}
\begin{theorem}
\label{thm:noiseless1}
If $\mL \in \R^{K \times K}$ is an SPSD matrix of rank $r$, then the matrix $\hat{\mL}$ output by the PLANS algorithm satisfies $\hat{\mL}=\mL$. Moreover, the number of oracle calls made by PLANS is at most  $K(r+1)$.
The sampling algorithm requires: $K + (K-1) + (K-2) + \ldots + (K - (r-1))+(K-r) \leq (r+1)K$ samples from the matrix $\mL$.
\end{theorem}
\begin{proof}
PLANS checks one column at a time starting from the second column, and uses the test in line $5$ to determine if the current column is independent of the previous columns. The validity of this test is guaranteed by proposition~\eqref{prop:degenerate1}. Each such test needs just one additional sample corresponding to the index $(c,c)$. If a column $c$ is found to be independent of the columns $1,2,\ldots,c-1$ then rest of the entries in column $c$ are queried. Notice, that by now we have already queried all the columns and rows of matrix $\mL$ indexed by the set $\cC$, and also queried the element $(c,c)$ in line $4$. Hence we need to query only $K-|\cC|-1$ more entries in column $c$ in order to have all the entries of column $c$. Combined with the fact that we query only $r$ columns completely and in the worst case all the diagonal entries might be queried, we get the total query complexity to be $(K-1)+ (K-2) + \ldots (K-r) + K\leq K(r+1)$.
\end{proof}
\section{Proof of Lemma 6.1}
We begin by stating the lemma.
\begin{lemma*}
\label{lem:bernstein_diff_n1}
Let $\hat{\mP}$ be a $p\times p$ random matrix that is constructed as follows. For each index $(i,j)$ independent of other indices, set $\hat{\mP}_{i,j}=\frac{H_{i,j}}{n_{i,j}}$, where $H_{i,j}$ is a random variable drawn from the distribution $\Bin(n_{i,j},p_{i,j})$. Let $Z=\hat{\mP}-\mP$. Then,
\begin{equation}
\label{eqn:ub}
||Z||_2 \leq \frac{2\log(2p/\delta)}{3~\underset{i,j}{\min}~ n_{i,j}}+\sqrt{\frac{\log(2p/\delta)}{2}\sum_{i,j}\frac{1}{n_{i,j}}}.
\end{equation}
Furthermore, if we denote by $\Delta$ the R.H.S. in Equation~\eqref{eqn:ub}, then $|\sigmamin(\hat{\mP})-\sigmamin(\mP)|\leq \Delta$. 
\end{lemma*}
\begin{proof}
Define, $\mS_{i,j}^{t}=\frac{1}{n_{i,j}} (X_{i,j}^t - p_{i,j}) \mE_{i,j}$, where $\mE_{i,j}$ is a $p\times p$ matrix with a $1$ in the $(i,j)^{\text{th}}$ entry and $0$ everywhere else, and $X^{t}_{i,j}$ is a random variable sampled from the distribution $\Bern(p_{i,j})$. If $X^{t}_{i,j}$ are independent for all $t,i,j$, then it is easy to see that $Z=\sum_{i,j}\frac{1}{n_{i,j}}\sum_{t=1}^{n_{i,j}} \mS^t_{i,j}$. Hence $\mS$ is a sum of independent random matrices and this allows to apply matrix Bernstein type inequalities. In order to apply the matrix Bernstein inequality, we would need upper bound on maximum spectral norm of the summands, and an upper bound on the variance of $\mZ$. We next bound these two quantities as follows,
\begin{equation}
||\mS^{t}_{i,j}||_2 = ||\frac{1}{n_{i,j}} (X_{i,j}^t - p_{i,j}) \mE_{i,j}||_2 = \frac{1}{n_{i,j}}|X_{i,j}^t - p_{i,j}|\leq \frac{1}{n_{i,j}}.
\end{equation}
To bound the variance of $\mZ$ we proceed as follows
\begin{equation}
\nu(Z)=||\underset{i,j}{\sum} \sum_{t=1}^{n_{i,j}}\bbE (\mS^{t}_{i,j})^\top \mS^{t}_{i,j}|| \bigwedge ||\underset{i,j}{\sum}\sum_{t=1}^{n_{i,j}} \bbE \mS^{t}_{i,j} (\mS^{t}_{i,j})^\top||
\end{equation}
Via elementary algebra and using the fact that $\Var(X_{i,j}^t)=p_{i,j}(1-p_{i,j})$It is easy to see that,
\begin{align}
\bbE (\mS^{t}_{i,j})^\top \mS^{t}_{i,j}&=\frac{1}{n_{i,j}^2}\bbE (X^t_{i,j}-p_{i,j})^2 (\mE_{i,j}^t)^\top \mE_{i,j}^t\\
&=\frac{1}{4n_{i,j}^2} \mE_{i,i}.
\end{align}
Using similar calculations we get $\bbE \mS^{t}_{i,j} (\mS^{t}_{i,j})^\top=\frac{1}{4n_{i,j}^2} \mE_{j,j}$. Hence, $\nu(Z)=\underset{i,j}{\sum} \sum_{t=1}^{n_{i,j}}\frac{1}{4n_{i,j}^2}=\sum_{i,j}\frac{1}{4n_{i,j}}$.
Applying matrix Bernstein, we get with probability at least $1-\delta$
\begin{equation}
||Z||_2 \leq \frac{2\log(2p/\delta)}{3~\underset{i,j}{\min}~ n_{i,j}}+\sqrt{\frac{\log(2p/\delta)}{2}\sum_{i,j}\frac{1}{n_{i,j}}}.
\end{equation}
The second part of the result follows immediately from Weyl's inequality which says that $|\sigmamin(\hat{\mP})-\sigmamin(\mP)| \leq ||\hat{\mP}-\mP||=||Z||$.
\end{proof}

\section{Sample complexity of successive elimination algorithm: Proof of Lemma 6.2}
\begin{lemma*}
\label{lem:se1}
The successive elimination algorithm shown in Figure~(6.2) on $m$ square matrices of size $\mA_1,\ldots,\mA_m$ each of size $p\times p$ outputs an index $\istar$ such that, with probability at least $1-\delta$, the matrix $\mA_{\istar}$ has the largest smallest singular value among all the input matrices. The total number of queries to the stochastic oracle are
\begin{equation}
 \sum_{k=2}^m O\left(\frac{p^3\log(2p\pi^2m^2/3\Delta_k^2\delta)}{\Delta_k^2}\right)+O\left(p^4\max_k \left(\frac{\log(2p\pi^2m^2/3\Delta_k^2\delta)}{\Delta_k^2}\right)\right)
\end{equation}
where $\Delta_{k,p}\defeq \max_{j=1,\ldots,m}\sigma_{\min}(\mA_j)-\sigma_{\min}(\mA_k)$
\end{lemma*}

\begin{proof}
Suppose matrix $\mA_1$ has the largest smallest singular value. 
From lemma~\eqref{lem:bernstein_diff_n1}, we know that with probability at least $1-\delta_t$,
$ |\sigma_{\min}(\bA_k)-\sigma_{\min}(\mA_k)|\leq \frac{2\log(2p/\delta_t)}{3~\underset{i,j}{\min}~ n_{i,j}(\mA)}+\sqrt{\frac{\log(2p/\delta_t)}{2}\sum_{i,j}\frac{1}{n_{i,j}(\mA)}}$. Hence, by union bound the probability that the matrix $\mA_1$ is eliminated in one of the rounds is at most $\sum_t\sum_{k=1}^m \delta_t\leq \sum_{t=1}^{\max}\sum_{k=1}^m \frac{6\delta}{\pi^2mt^2}=\delta$. This proves that the successive elimination step identifies the matrix with the largest smallest singular value.

 An arm $k$ is eliminated in round $t$ if $\alpha_{t,1}+\alpha_{t,k}\leq \hat{\sigma}^{\max}_t-\sigma_{\min}(\bA_k)$. By definition,
\begin{equation}
\Delta_{k,p} - (\alpha_{t,1}+\alpha_{t,k})=(\sigma_{\min}(\mA_1)-\alpha_{t,1})-(\sigma_{\min}(\mA_k)+\alpha_{t,k})\geq \sigma_{\min}(\hat{\mA_1})-\sigma_{\min}(\mA_k)\geq \alpha_{t,1}+\alpha_{t,k}
\end{equation}
That is if $\alpha_{t,1}+\alpha_{t,k}\leq \frac{\Delta_{k,p}}{2}$, then arm $k$ is eliminated in round $t$. 
By construction, since in round $t$ each element in each of the surviving set of matrices has been queried at least $t$ times, we can say that
$\alpha_{t,j}\leq \frac{2\log(2p/\delta_t)}{3t}+\sqrt{\frac{p^2\log(2p/\delta_t)}{2t}}$
for any index $j$ corresponding to the set of surviving arms. Hence arm $k$ gets eliminated after
\begin{equation}
t_k =O\left(\frac{p^2\log(2p\pi^2m^2/3\Delta_{k,p}^2\delta)}{\Delta_{k,p}^2}\right)
\end{equation}
In each round $t$ the number of queries made are $O(p)$ for each of the $m$ matrices corresponding to the row and column which is different among them, and $O(p^2)$ corresponding to the left $p-1\times p-1$ submatrix that is common to all of the matrices $A_1,\ldots,A_m$. Hence, the total number of queries to the stochastic oracle is 
\begin{equation*}
p\sum_{k=2}^m t_k + p^2 \max_k t_k= \sum_{k=2}^m O\left(\frac{p^3\log(2p\pi^2m^2/3\Delta_{k,p}^2\delta)}{\Delta_{k,p}^2}\right)+O\left(p^4\max_k \left(\frac{\log(2p\pi^2m^2/3\Delta_{k,p}^2\delta)}{\Delta_{k,p}^2}\right)\right)\qedhere
\end{equation*}
\end{proof}

\section{Proof of Nystrom method}
In this supplementary material we provide a proof of Nystrom extension in $\max$ norm when we use a stochastic oracle to obtain estimators  $\bC,\bW$ of the matrices $\mC, \mW$. The question that we are interested in is how good is the estimate of the Nystrom extension obtained using matrices $\bC, \bW$ w.r.t. the Nystrom extension obtained using matrices $\mC,\mW$. This is answered in the theorem below.
\begin{theorem}
\label{thm:nystrom_maxnorm_approx1}
Suppose the matrix $\mW$ is an invertible $r\times r$ matrix. Suppose, by multiple calls to a stochastic oracle we construct estimators $\bC,\bW$ of $\mC,\mW$. Now, consider the matrix $\bC\bW^{-1}\bC^{\top}$ as an estimate $\mC\mW^{-1}\mC^{\top}$. Given any $\delta\in (0,1)$, with probability atleast $1-\delta$, 
\begin{equation*}
\lV \mC \mW^{-1} \mC^\top - \bC \bW^{-1} \bC^\top \rV_{\max}\leq \epsilon
\end{equation*}
after making $M$ number of oracle calls to a stochastic oracle, where 
\begin{align*}
M &\geq 100C_1(W,C)\log(2Kr/\delta)\max\left(\frac{Kr^{7/2}}{\epsilon},\frac{Kr^3}{\epsilon^2}\right)+200 C_2(W,C)\log(2r/\delta)\max\left(\frac{r^5}{\epsilon},\frac{r^7}{\epsilon^2}\right)
\end{align*}
where $C_1(\mW,\mC)$ and $C_2(\mW,\mC)$ are given by the following equations
\begin{align*}
C_1(\mW,\mC)&=\max\left(\lV \mW^{-1}\mC^\top\rV_{\max},\lV \mW^{-1}\mC^\top\rV_{\max}^2,\lV \mW^{-1}\rV_{\max},\lV\mC\mW^{-1}\rV_1^2,\lV\mW^{-1}\rV_2\lV\mW^{-1}\rV_{\max}\right)\\
C_2(\mW,\mC)&=\max\left(\lV \mW^{-1}\rV_2^2 \lV\mW^{-1}\rV_{\max}^2,\lV \mW^{-1}\rV_2 \lV\mW^{-1}\rV_{\max},\lV\mW^{-1}\rV_2,\lV\mW^{-1}\rV_2^2\right)
\end{align*}
\end{theorem}

Our proof proceeds by a series of lemmas, which we state next.
\begin{lemma}
\label{lem:basic}
\begin{equation*}
\lV \mC \mW^{-1} \mC^\top - \bC \bW^{-1} \bC^\top \rV_{\max} \leq \lV  (\mC - \bC) \mW^{-1} \mC^\top \rV_{\max} + \lV \bC \bW^{-1} (\mC  - \bC)^\top \rV_{\max} + \lV \bC ( \mW^{-1}  - \bW^{-1}) \mC^\top  \rV_{\max}
\end{equation*}
\end{lemma}
\begin{proof}
\begin{align*}
\lV \mC \mW^{-1} \mC^\top - \bC \bW^{-1} \bC^\top \rV_{\max} &=  \lV \mC \mW^{-1} \mC^\top - \bC \mW^{-1} \mC^\top + \bC \mW^{-1} \mC^\top   - \bC \bW^{-1} \bC^\top \rV_{\max} \\
&\leq  \lV \mC \mW^{-1} \mC^\top - \bC \mW^{-1} \mC^\top \rV_{\max}+ \lV \bC \mW^{-1} \mC^\top   - \bC \bW^{-1} \bC^\top \rV_{\max} \\
&=  \lV \mC \mW^{-1} \mC^\top - \bC \mW^{-1} \mC^\top \rV_{\max} +\\
&\hspace{15pt} \lV \bC \mW^{-1} \mC^\top - \bC \bW^{-1} \mC^\top + \bC \bW^{-1} \mC^\top  - \bC \bW^{-1} \bC^\top \rV_{\max}\\
&\leq  \lV \mC \mW^{-1} \mC^\top - \bC \mW^{-1} \mC^\top \rV_{\max} + \lV \bC W^{-1} \mC^\top - \bC \bW^{-1} \mC^\top \rV_{\max} +\\
&\hspace{15pt} \lV \bC \bW^{-1} \mC^\top  - \bC \bW^{-1} \bC^\top \rV_{\max}\\
&=  \lV  (\mC - \bC) \mW^{-1} \mC^\top \rV_{\max} + \lV \bC \bW^{-1} (\mC  - \bC)^\top \rV_{\max} + \lV \bC ( \mW^{-1}  - \bW^{-1}) \mC^\top  \rV_{\max}\\
\end{align*}
\end{proof}
In the following lemmas we shall bound the three terms that appear in the R.H.S of the bound of Lemma~\eqref{lem:basic}.
\begin{lemma}
\begin{equation}
 \lV  (C - \hat{C}) W^{-1} C^\top \rV_{\max} \leq \frac{2||\mW^{-1} \mC^\top||_{\max}}{3m}\log(2Kr/\delta)+\sqrt{\frac{r~||\mW^{-1} \mC^\top||_{\max}^2 \log(2Kr/\delta)}{2m}}
 \end{equation}
 \end{lemma}
 \begin{proof}
Let $\mM = \mW^{-1} \mC^\top$, then $ \lV  (\mC - \bC) \mW^{-1} \mC^\top \rV_{\max} =  \lV  (\mC - \bC) \mM \rV_{\max}$. By the definition of max norm we have 
\begin{equation*}
\lV(\mC - \bC) M \rV_{\max} = \max_{i,j}\left|\sum_{p=1}^l(\mC-\bC)_{i,p} \mM_{p,j}\right|
\end{equation*}
Fix a pair of indices $(i,j)$, and consider the expression $\left|\sum_{p=1}^l(\mC-\bC)_{i,p} \mM_{p,j}\right|$

Define $r_{i,p}=(\mC-\bC)_{i,p}$. By definition of $r_{i,p}$ we can write $r_{i,p}=\frac{1}{m} \sum_{t=1}^m r_{i,p}^t$, where $r_{i,p}^t$ are a set of independent random variables with mean 0 and variance at most $1/4$. This decomposition combined with scalar Bernstein inequality gives that with probability at least $1-\delta$
\begin{align*}
\left|\sum_{p=1}^l (\bC-\bC)_{i,p} \mM_{p,j}\right|&=\left|\sum_{p=1}^l r_{i,p} \mM_{p,j}\right|\\
&=\left|\sum_{p=1}^l \sum_{t=1}^m \frac{1}{m}r_{i,p}^t \mM_{p,j}\right|\\
&\leq \frac{2||M||_{\max}}{3m}\log(2/\delta)+\sqrt{\frac{r~||M||_{\max}^2 \log(2/\delta)}{2m}}
\end{align*}
Applying a union bound over all possible $Kr$ choices of index pairs $(i,j)$, we get the desired result.
\end{proof}
Before we establish bounds on the remaining two terms in the RHS of Lemma~\eqref{lem:basic} we state and prove a simple proposition that will be used at many places in the rest of the proof.
\begin{proposition}
For any two real matrices $M_1 \in \R^{n_1 \times n_2}, M_2 \in \R^{n_2 \times n_3}$ the following set of inequalities are true:
\begin{enumerate}
\item $\lV \mM_1 \mM_2 \rV_{\max} \leq  \lV \mM_1 \rV_{\max} \lV \mM_2 \rV_1$ 
 \item $\lV \mM_1 \mM_2 \rV_{\max} \leq  \lV \mM_1^\top \rV_1 \lV \mM_2 \rV_{\max}$ 
 \item $\lV \mM_1 \mM_2 \rV_{\max} \leq  \lV \mM_1 \rV_2 \lV \mM_2 \rV_{\max}$
 \item $\lV \mM_1 \mM_2 \rV_{\max} \leq  \lV \mM_2 \rV_2 \lV \mM_1 \rV_{\max}$
\end{enumerate}
where, the $\lV \cdot \rV_p$ is the induced $p$ norm.
\label{prop:matinfbound1}
\end{proposition}
\begin{proof}
Let $\ve_i$ denote the $i^{\text{th}}$ canonical basis vectors in $\bbR^K$. We have,
\begin{align*}
\lV \mM_1 \mM_2 \rV_{\max} &= \max_{i,j} \lv e_i^\top \mM_1\mM_2 e_j \rv \\
&\leq \max_{i,j} \lV e_i^\top \mM_1 \rV_{\max} \lV \mM_2 e_j \rV_1 \\ 
&= \max_{i} \lV e_i^\top \mM_1 \rV_{\max} \max_{i} \lV \mM_2 e_j \rV_1 \\ 
&=  \lV \mM_1 \rV_{\max} \lV \mM_2 \rV_1. 
\end{align*}
To obtain the first inequality above we used Holder's inequality and the last equality follows from the definition of $||\cdot||_1$ norm.
To get the second inequality, we use the observations that  $\lV \mM_1 \mM_2 \rV_{\max} = \lV \mM_2^\top \mM_1^\top \rV_{\max}$. Now applying the first inequality to this expression we get the desired result. Similar techniques yield the other two inequalities.
\end{proof} 
\begin{lemma}
\label{lem:T2}
With probability at least $1-\delta$, we have
\begin{align*}
\lV \bC \bW^{-1} (\mC  - \bC)^\top \rV_{\max} &\leq  \frac{r^2}{2m} \left(\lV \bW^{-1} -\mW^{-1}\rV_{\max}+\lV \mW^{-1}\rV_{\max}\right) \log (2Kr/\delta) + \\
&\hspace{10pt} r^2 \lV \bW^{-1} - \mW^{-1} \rV_{\max} \sqrt{\frac{\log (2Kr/\delta)}{2m}} +
 r \lV \mC \mW^{-1} \rV_1 \sqrt{\frac{\log (2Kr/\delta)}{2m}}
\end{align*}
\end{lemma}
\begin{proof}
\begin{align}
\lV \bC \bW^{-1} (\mC  - \bC)^\top \rV_{\max} &\leq \lV (\bC \bW^{-1} - \mC \mW^{-1}+ \mC \mW^{-1}) (\mC  - \bC)^\top \rV_{\max}\nonumber \\
&\leqa \lV (\bC \bW^{-1} - \mC \mW^{-1})(\mC  - \bC)^\top \rV_{\max} + \lV \mC \mW^{-1} (\mC  - \bC)^\top \rV_{\max}\nonumber \\
&\leqb \lV \bC \bW^{-1} - \mC \mW^{-1} \rV_{\max} \lV (\mC  - \bC)^\top \rV_1 + \lV \mC \mW^{-1} \rV_{\max} \lV (\mC  - \bC)^\top \rV_1 \label{eqn:temp1}
\end{align}
To obtain inequality (a) we used triangle inequality for matrix norms, and to obtain inequality (b) we used Proposition~\eqref{prop:matinfbound1}. We next upper bound the first term in the R.H.S. of Equation~\eqref{eqn:temp1}.

We bound the term $\lV \bC \bW^{-1} - \mC \mW^{-1} \rV_{\max}$ next.
\begin{align}
\lV \bC \bW^{-1} - \mC \mW^{-1} \rV_{\max} &\leq \lV \bC \bW^{-1} - \mC \bW^{-1} + \mC \bW^{-1}  - \mC \mW^{-1} \rV_{\max}  \nonumber \\
&\leq \lV \bC \bW^{-1} - \mC \bW^{-1}\rV_{\max}   + \lV \mC \bW^{-1}  - \mC \mW^{-1} \rV_{\max} \nonumber \\
&=  \lV (\bC - \mC) \bW^{-1} \rV_{\max}   + \lV \mC ( \bW^{-1}  - \mW^{-1}) \rV_{\max}  \nonumber \\
&\leqa \lV (\bC - \mC)^\top \rV_1 \lV \bW^{-1} \rV_{\max}   + \lV \mC^\top \rV_1 \lV \bW^{-1}  - \mW^{-1} \rV_{\max} \label{eqn:temp2}
\end{align}
We used Proposition~\eqref{prop:matinfbound1} to obtain inequality (a).
Combining Equations~\eqref{eqn:temp1} and ~\eqref{eqn:temp2} we get,
\begin{align}
 \lV \bC \bW^{-1} (\mC  - \bC)^\top \rV_{\max} &\leq \lV (\bC - \mC)^\top \rV_1 \left(  \lV (\bC - \mC)^\top \rV_1 \lV \bW^{-1} \rV_{\max}   + \lV \mC^\top \rV_1 \lV \bW^{-1}  - \mW^{-1} \rV_{\max}  +   \lV \mC \mW^{-1} \rV_{\max} \right) \nonumber\\
&= \lV (\bC - \mC)^\top \rV_1^2  \lV \bW^{-1} \rV_{\max} + \lV (\bC - \mC)^\top \rV_1  \lV \mC^\top \rV_1 \lV \bW^{-1}  - \mW^{-1} \rV_{\max}  +\nonumber\\
&\hspace{15pt}  \lV (\bC - \mC)^\top \rV_1 \lV \mC \mW^{-1} \rV_{\max} \label{eqn:temp3} 
\end{align}
Since all the entries of the matrix $\mC$ are probabilities we have $\lV \mC \rV_{\max} \leq 1$ and $\lV \mC^\top\rV_1\leq r$. Moreover, since each entry of the matrix $\hat{\mC}-\mC$ is the average of $m$ independent random variables with mean $0$, and each bounded between $[-1,1]$, by Hoeffding's inequality and union bound, we get that with probability at least $1-\delta$ 
\begin{equation}
\label{eqn:temp4}
\lV (\bC - \mC)^\top \rV_1 \leq r\sqrt{\frac{\log (2Kr/\delta)}{2m}}
\end{equation}
\end{proof}
The next proposition takes the first steps towards obtaining an upper bound on  $\lV \hat{C} ( W^{-1}  - \hat{W}^{-1}) C^\top  \rV_{\max}$
\begin{proposition}
\label{prop:silly}
\begin{equation*}
\lV \bC ( \mW^{-1}  - \bW^{-1}) \mC^\top  \rV_{\max}\leq \min \left\{ r^2 \lV  \mW^{-1}  - \bW^{-1} \rV_{\max} , r \lV  \mW^{-1}  - \bW^{-1} \rV_1 \right\}
\end{equation*}
\end{proposition}
\begin{proof}
\begin{align}
\lV \bC ( \mW^{-1}  - \bW^{-1}) \mC^\top  \rV_{\max} &\leqa \lV \bC ( \mW^{-1}  - \bW^{-1}) \rV_{\max} \lV \mC^\top  \rV_1 \nonumber\\	
&\leqb r  \lV \bC ( \mW^{-1}  - \bW^{-1}) \rV_{\max} \nonumber\\
&\leqc \min \left\{ r^2 \lV  \mW^{-1}  - \bW^{-1} \rV_{\max} , r \lV  \mW^{-1}  - \bW^{-1} \rV_1 \right\}\label{eqn:temp5}
\end{align}
In the above bunch of inequalities (a) and (c) we used Proposition~\eqref{prop:matinfbound1} and to obtain inequality (b) we used the fact that $||\mC^\top||_{\max}\leq r$.
\end{proof}
Hence, we need to bound $ \lV  \mW^{-1}  - \bW^{-1} \rV_{\max}$ and $ \lV  \mW^{-1}  - \bW^{-1} \rV_1$.

Let us define $\hat{\mW} = \mW + \mE_W$ where $\mE_W$ is the error-matrix and $\hat{\mW}$ is the sample average of $m$ independent samples of a random matrix where $\bbE \hat{\mW}_k(i,j) = \mW(i,j)$.

\begin{lemma}
\label{lem:taylor}
Let us define $\hat{\mW} - \mW = \mE_W$. Suppose, $\lV \mW^{-1} \mE_W \rV_2 \leq \frac{1}{2}$,  then $$\lV \hat{\mW}^{-1} - \mW^{-1} \rV_{\max} \leq 2\lV \mW^{-1}\rV_2 \lV \mE_W\rV_2 \lV \mW^{-1}\rV_{\max}$$.
\end{lemma}
\begin{proof}
Since $\lV \mW^{-1} \mE_W \rV_2 < 1$, we can apply the Taylor series expansion:
$$ (\mW + \mE_W)^{-1} = \mW^{-1} - \mW^{-1} \mE_W \mW^{-1} + \mW^{-1} \mE_W \mW^{-1} \mE_W \mW^{-1} + \cdots $$
Therefore:
\begin{align*}
\lV \hat{\mW}^{-1} - \mW^{-1} \rV_{\max} &= \lV \mW^{-1} - \mW^{-1} \mE_W \mW^{-1} + \mW^{-1} \mE_W \mW^{-1} \mE_W \mW^{-1} + \cdots - \mW^{-1}  \rV_{\max}\\
&\leqa \lV \mW^{-1} \mE_W \mW^{-1}  \rV_{\max} + \lV \mW^{-1} \mE_W \mW^{-1} \mE_W \mW^{-1}  \rV_{\max} + \cdots \\
&\leqb \lV \mW^{-1} \mE_W \rV_2 \lV W^{-1}\rV_{\max}+\lV \mW^{-1} \mE_W \rV_2^2 \lV W^{-1}\rV_{\max}+\ldots\\
&\leqc2\lV \mW^{-1}\rV_2 \lV \mE_W\rV_2 \lV \mW^{-1}\rV_{\max}
\end{align*}
To obtain the last inequality we used the hypothesis of the lemma, and to obtain inequality (a) we used the triangle inequality for norms, and to obtain inequality (b) we used proposition~\eqref{prop:matinfbound1}. Inequlaity (c) follows from the triangle inequality.
\end{proof}
Thanks to Lemma~\eqref{lem:taylor} and proposition~\eqref{prop:silly} we know that $\lV \bC ( \mW^{-1}  - \bW^{-1}) \mC^\top  \rV_{\max} \leq r^2\epsilon$. We now need to guarantee that the hypothesis of lemma~\eqref{lem:taylor} applies. The next lemma helps in doing that.
\begin{lemma}
\label{lem:bernstein_Winv}
With probability at least $1-\delta$ we have 
\begin{equation}
\lV \mE_W\rV=\lV \hat{\mW}-\mW\rV\leq \frac{2r}{3m}\log(2r/\delta)+\sqrt{\frac{r\log(2r/\delta)}{2m}}
\end{equation}
\end{lemma}
\begin{proof}
The proof is via matrix Bernstein inequality. 
By the definition of $\hat{\mW}$, we know that $\hat{\mW}-\mW=\frac{1}{m}\sum (\mW_i -\mW)$, where $\hat{\mW}$ is $0-1$ random matrix where the $(i,j)^{\text{th}}$ entry of the matrix $\hat{\mW}$ is a single Bernoulli sample sampled from $\Bern(\mW_{i,j})$. For notational convenience denote $Z_i\defeq \frac{1}{m}\hat{\mW}_i -\mW$ This makes $\hat{\mW}-\mW=\frac{1}{m}\sum \mW_i -\mW$ an average of $m$ independent random matrices each of whose entry is a $0$ mean random variable with variance at most $1/4$, with each entry being in $[-1,1]$. In order to apply the matrix Bernstein inequality we need to upper bound $\nu,L$ (see Theorem~\eqref{thm:matBern}), which we do next.
\begin{align}
\lV\frac{1}{m} (\hat{\mW}_i-\mW)\rV_2\leq \frac{1}{m}\sqrt{r^2}=\frac{r}{m}.
\end{align}
In the above inequality we used the fact that each entry of $(\hat{\mW}_i-\mW)$ is between $[-1,1]$ and hence the spectral norm of this matrix is at most $\sqrt{r^2}$.
We next bound the parameter $\nu$.
\begin{align}
\label{eqn:nu_Z}
\nu &= \frac{1}{m^2}\max \left\{ \left\Vert \sum_i \bbE \mZ_i \mZ_i^\top  \right\Vert, \left\Vert \sum_i \bbE \mZ_i^\top \mZ_i  \right\Vert \right\} 
\end{align}
It is not hard to see that the matrix $\bbE \mZ_i\mZ_i^\top$ is a diagonal matrix, where each diagonal entry is at most $\frac{l}{4}$. The same holds true for $\bbE \mZ_i\mZ_i^\top$. Putting this back in Equation~\eqref{eqn:nu_Z} we get $\nu \leq \frac{r}{4m}$.
Putting $L=\frac{r}{m}$ and $\nu= \frac{r}{4m}$, we get
\begin{equation}
\lV \hat{\mW}-\mW\rV\leq \frac{2r}{3m}\log(2r/\delta)+\sqrt{\frac{r\log(2r/\delta)}{2m}}
\end{equation}
\end{proof}
We are now ready to establish the following bound
\begin{lemma}
\label{lem:some_lem}
Assuming that $m\geq m_0\defeq \frac{4r\lV \mW^{-1}\rV}{3}+2r\log(2r/\delta)\lV\mW^{-1}\rV_2^2$, with probability at least $1-\delta$ we will have
\label{lem:T3}
\begin{equation}
\lV \bC ( \mW^{-1}  - \bW^{-1}) \mC^\top  \rV_{\max}\leq 2r^2 \lV\mW^{-1}\rV_2 \lV \mW^{-1}\rV_{\max}\left(\frac{2r}{3m}\log(2r/\delta)+\sqrt{\frac{r\log(2r/\delta)}{2m}}\right).
\end{equation}
\end{lemma}
\begin{proof}
\begin{align*}
\lV \bC ( \mW^{-1}  - \bW^{-1}) \mC^\top  \rV_{\max}&\leqa  r^2 \lV\mW^{-1}-\hat{\mW}^{-1}\rV_{\max}\\
&\leqb 2r^2 \lV\mW^{-1}\mE_{\mW}\rV_2 \lV \mW^{-1}\rV_{\max}\\
&\leqc 2r^2 \lV\mW^{-1}\rV_2\lV\mE_{\mW}\rV_2 \lV \mW^{-1}\rV_{\max}\\
&\leqd 2r^2 \lV\mW^{-1}\rV_2\lV \mW^{-1}\rV_{\max}\left(\frac{2r}{3m}\log(2r/\delta)+\sqrt{\frac{r\log(2r/\delta)}{2m}}\right)\qedhere
\end{align*}
To obtain inequality (a) above we used proposition~\eqref{prop:silly}, to obtain inequality (b) we used lemma~\eqref{lem:taylor}, and finally to obtain inequality (c) we used the fact that matrix 2-norms are submultiplicative.
\end{proof}
With this we now have bounds on all the necessary quantities. The proof of our theorem essentially requires us to put all these terms together.

\section{Proof of Theorem~\eqref{thm:nystrom_maxnorm_approx1}}
Since we need the total error in $\max$ norm to be at most $\epsilon$, we will enforce that each term of our expression be atmost $\frac{\epsilon}{10}$.
From lemma~\eqref{lem:basic} we know that the maxnorm is the sum of three terms. Let us call the three terms in the R.H.S. of Lemma~\eqref{lem:basic} $T_1, T_2, T_3$ repsectively. We then have that if we have $m_1$ number of copies of the matrix $\mC$, where 
 \begin{equation}
m_1\geq \frac{20\lV \mW^{-1}\mC^\top\rV_{\max} \log(2Kr/\delta)}{3\epsilon} \bigwedge \frac{100r \lV \mW^{-1}\mC^\top\rV_{\max}^2 \log(2Kr/\delta)}{2\epsilon^2}
\end{equation}
then $T_1\leq \epsilon/5$. Next we look at $T_3$. From lemma~\eqref{lem:T3} it is easy to see that we need $m_3$ independent copies  of the matrix  $\mW$ so that $T_3\leq \epsilon/5$, where $m_3$ is equal to 
\begin{equation}
m_3\geq \frac{40r^3\lV\mW^{-1}\rV_2\lV\mW^{-1}\rV_{\max}\log(2r/\delta)}{3\epsilon} \bigwedge \frac{400r^5\lV\mW^{-1}\rV_2^2\lV\mW^{-1}\rV_{\max}^2\log(2r/\delta)}{2\epsilon^2}
\end{equation}
Finally we now look at $T_2$. Combining lemma~\eqref{lem:T2}, and lemma~\eqref{lem:taylor} and ~\eqref{lem:bernstein_Winv} and after some elementary algebraic calculations we get that we need $m_2$ independent copies of the matrix $\mC$ and $\mW$ to get $T_2\leq \frac{3\epsilon}{5}$, where $m_2$ is 
\begin{equation}
m_2\geq 100\max(\lV\mW^{-1}\rV_{\max},\lV \mC\mW^{-1}\rV_1^2,\lV\mW^{-1}\rV_2\lV\mW^{-1}\rV_{\max})\log(2Kr/\delta)\left(\frac{r^{5/2}}{\epsilon}, \frac{r^2}{\epsilon^2}\right)
\end{equation}
The number of calls to stochastic oracle is $r^2(m_0+m_3)+ Kr(m_1+m_2)$, where  $m_0$ is the number as stated in Lemma~\eqref{lem:some_lem}. Using the above derived bounds for $m_0+m_1,m_2,m_3$ we get
\begin{align*}
Kr(m_1+m_2)+r^2(m_0+m_3) &\geq 100\log(2Kr/\delta)C_1(W,C)\max\left(\frac{Kr^{7/2}}{\epsilon},\frac{Kr^3}{\epsilon^2}\right)+\\
& \hspace{20pt}200 C_2(W,C)\log(2r/\delta)\max\left(\frac{r^5}{\epsilon},\frac{r^7}{\epsilon^2}\right)
\end{align*}
where $C_1(\mW,\mC)$ and $C_2(\mW,\mC)$ are given by the following equations
\begin{align*}
C_1(\mW,\mC)&=\max\left(\lV \mW^{-1}\mC^\top\rV_{\max},\lV \mW^{-1}\mC^\top\rV_{\max}^2,\lV \mW^{-1}\rV_{\max},\lV\mC\mW^{-1}\rV_1^2,\lV\mW^{-1}\rV_2\lV\mW^{-1}\rV_{\max}\right)\\
C_2(\mW,\mC)&=\max\left(\lV \mW^{-1}\rV_2^2 \lV\mW^{-1}\rV_{\max}^2,\lV \mW^{-1}\rV_2 \lV\mW^{-1}\rV_{\max},\lV\mW^{-1}\rV_2,\lV\mW^{-1}\rV_2^2\right)
\end{align*}
\section{Additional experimental results: Comparison with LRMC on Movie Lens datasets}
First we present the results on the synthetic dataset. To generate a low-rank matrix, we take a random matrix in $\mL_1 = [0,1]^{K \times r}$ and then define $\mL_2 = \mL_1 \mL_1^\top$. Then get $\mL = \mL_2/\max_{i,j}(\mL_2)_{i,j}$. This matrix $\mL$ will be $K \times K$ and have rank $r$.
\begin{figure}[]
\begin{subfigure}{0.5\textwidth}
\centering	
\includegraphics[width=\linewidth]{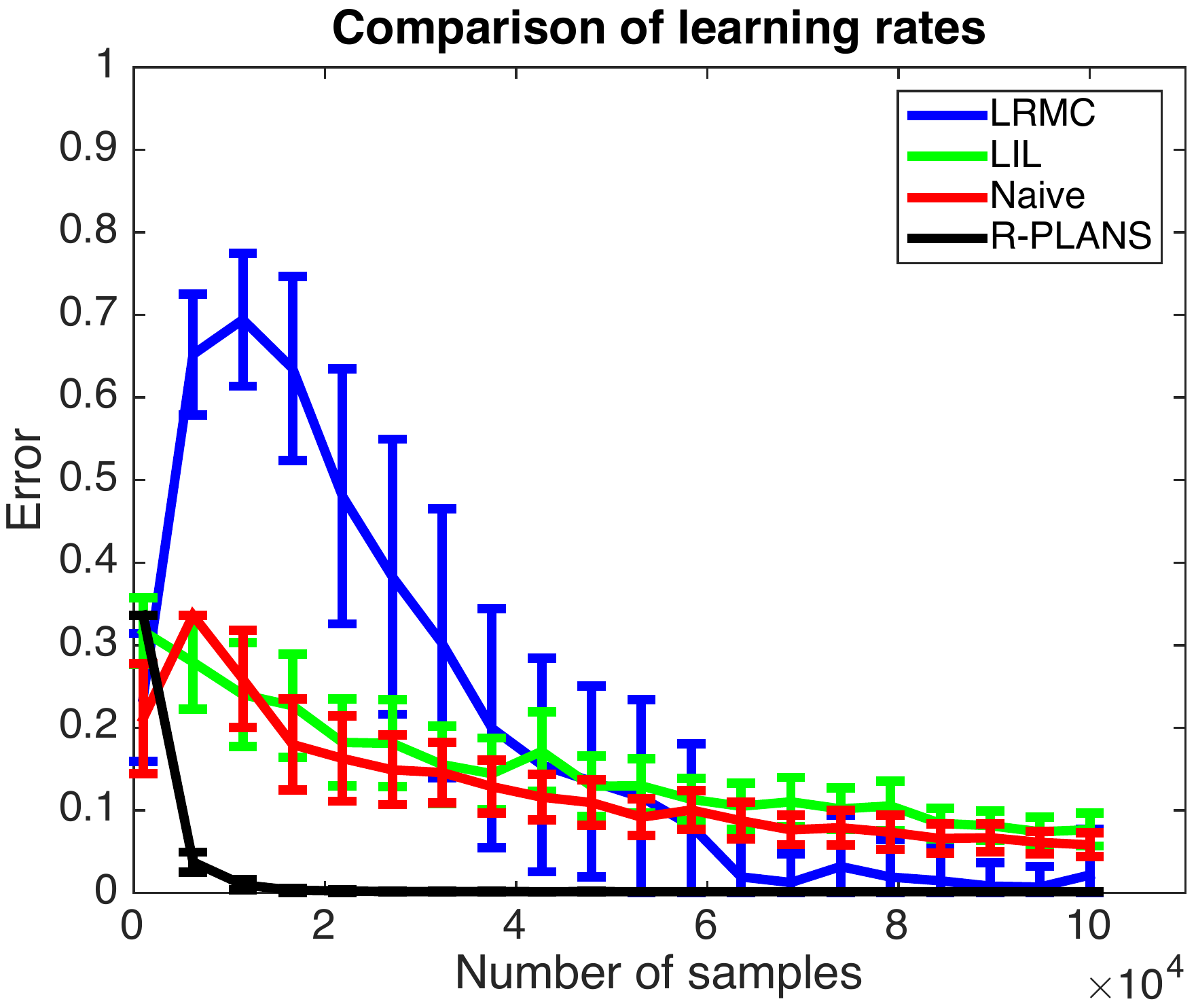}
\caption{ML-100K; $(200, 2)$\label{fig:syn_a}}
\end{subfigure}%
\begin{subfigure}{0.5\textwidth}
\centering
\includegraphics[width=\linewidth]{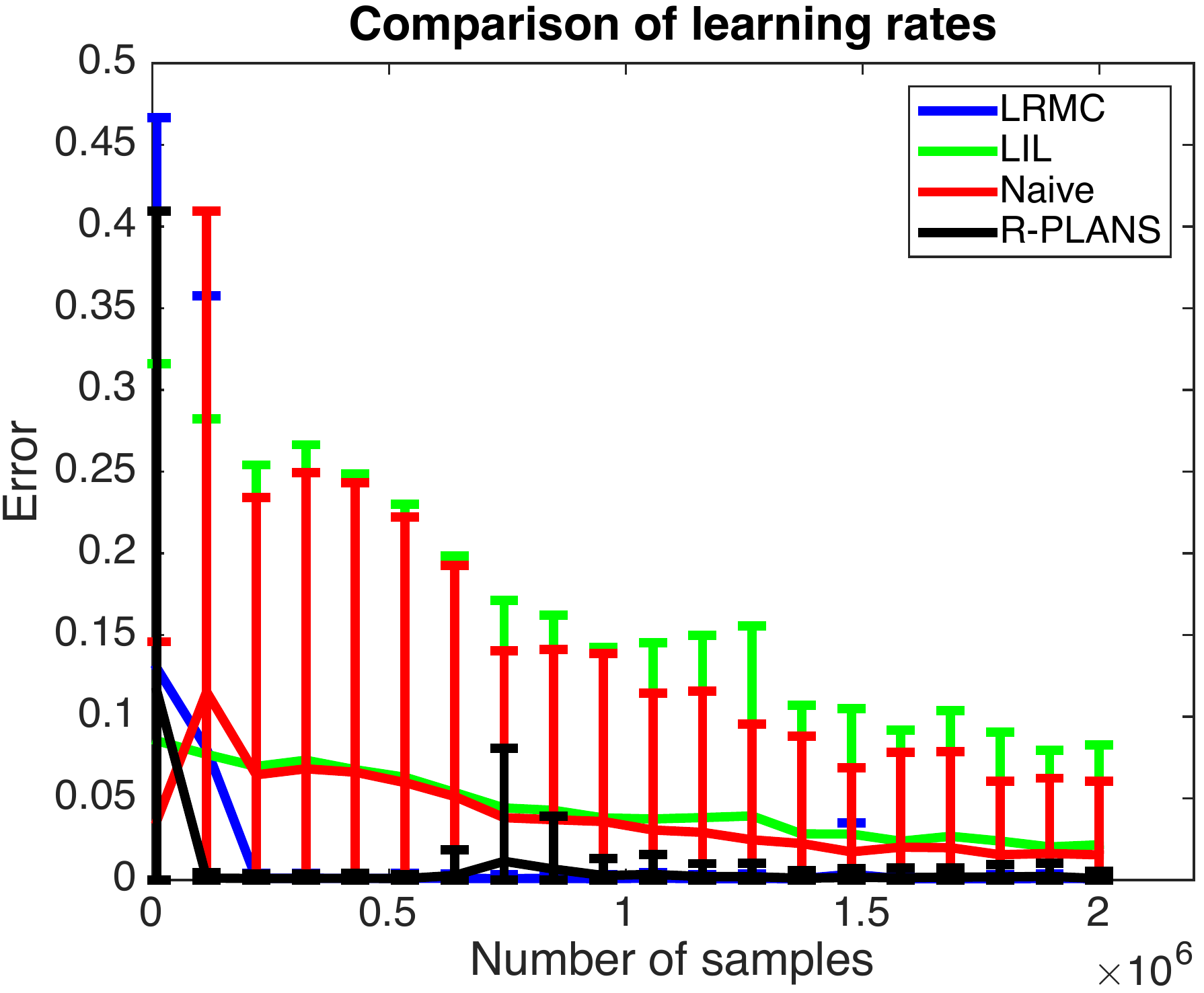}
\caption{ML-100K; $(200, 4)$\label{fig:syn_b}}
\end{subfigure}%
\caption{\label{expts_syn}\footnotesize{Error of various algorithms with increasing budget. Numbers in the brackets represent values for $(K,r)$. The error is defined as $\mL_{\hat{i},\hat{j}}-\mL_{i_\star,j_\star}$ where $(\hat{i}, \hat{j})$ is a pair of optimal choices as estimated by each algorithm.}}
\end{figure}

In Figure \ref{expts_real}, you can find the comparison of LRMC and R-PLANS on the ML-100K dataset. 
\begin{figure}[]
\begin{subfigure}{0.5\textwidth}
\centering	
\includegraphics[width=\linewidth]{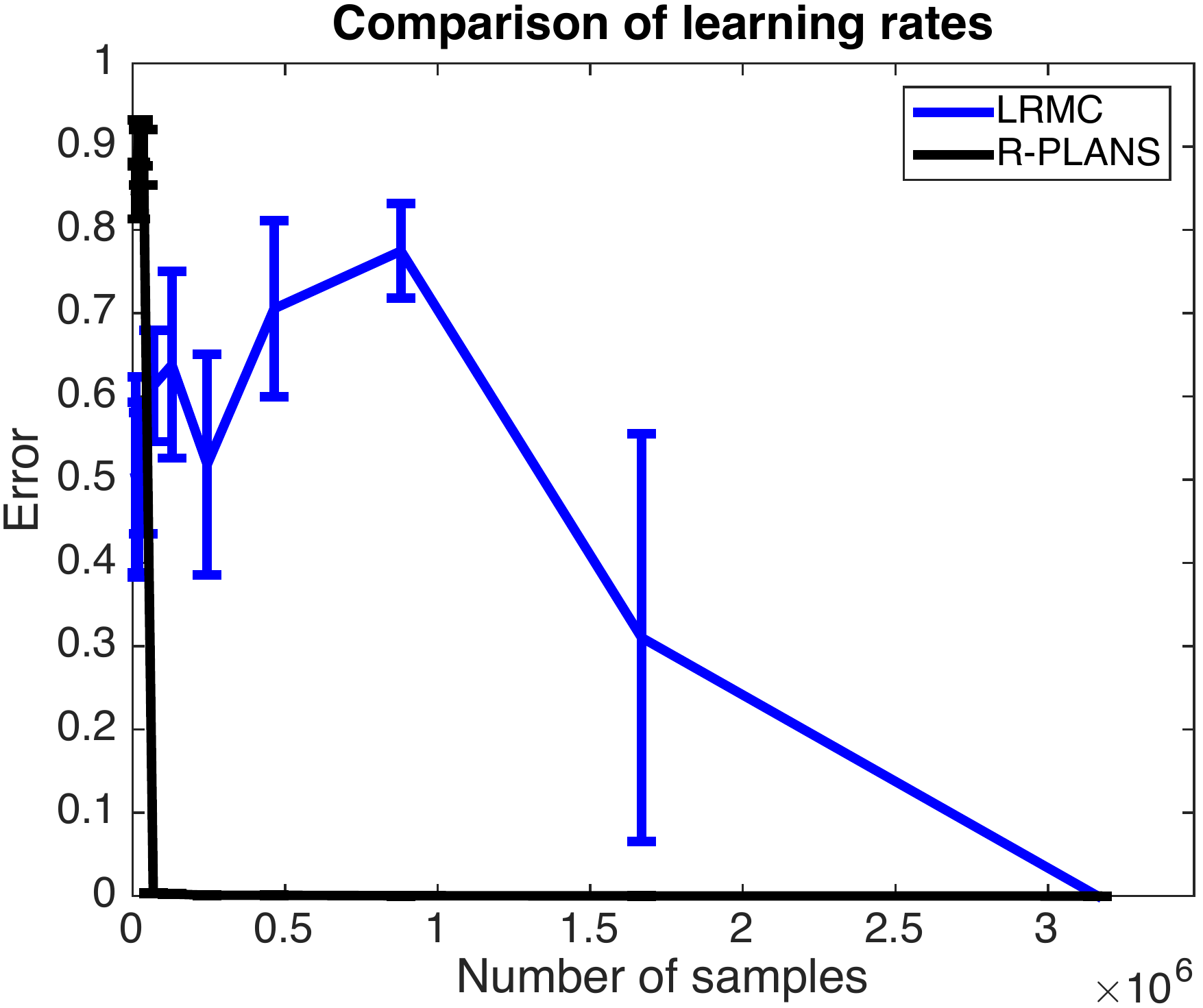}
\caption{ML-100K; $(200, 2)$\label{fig:real_a1}}
\end{subfigure}%
\begin{subfigure}{0.5\textwidth}
\centering
\includegraphics[width=\linewidth]{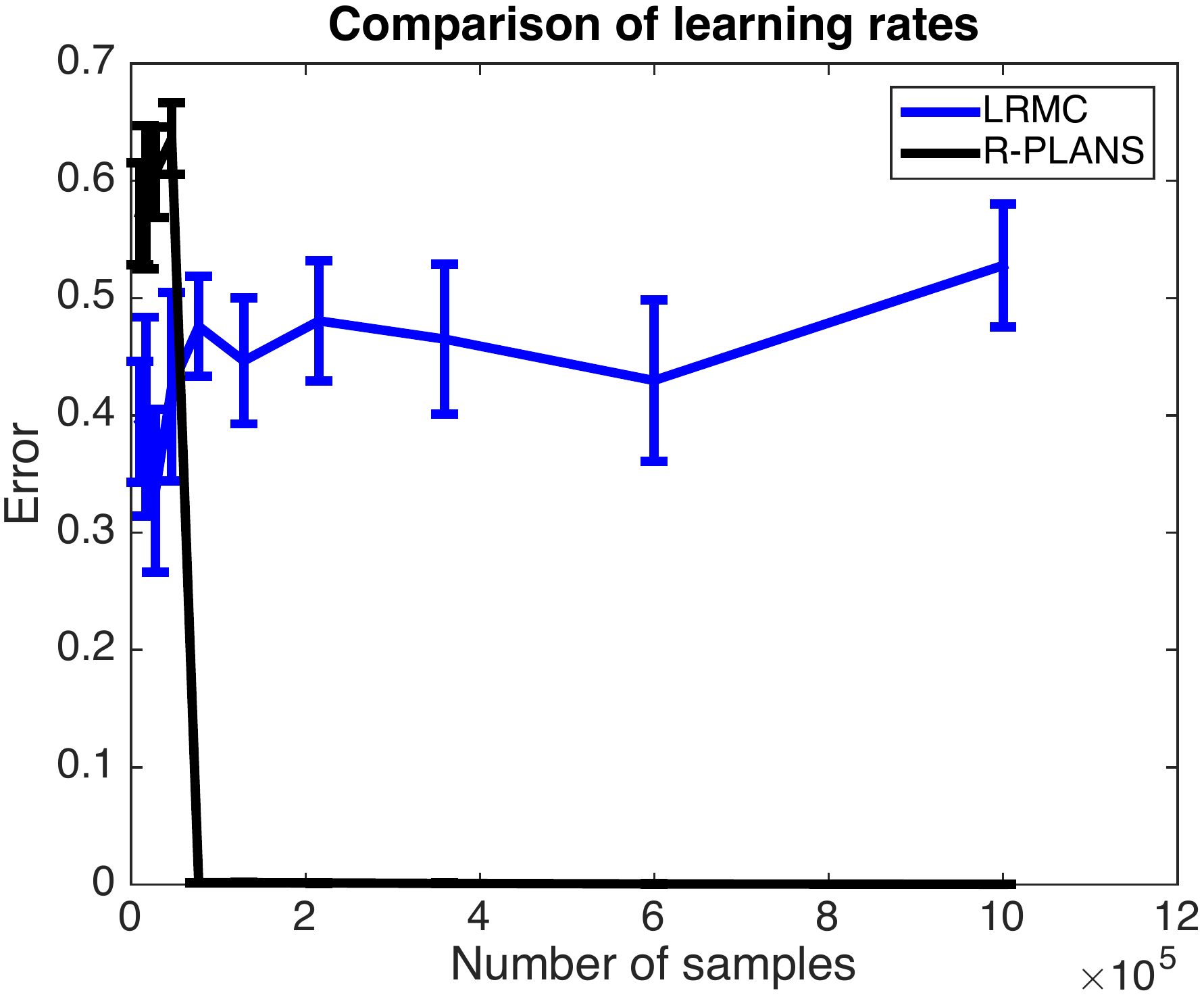}
\caption{ML-100K; $(200, 4)$\label{fig:real_b1}}
\end{subfigure}%
\caption{\label{expts_real}\footnotesize{Error of LRMC and R-PLANS algorithms with increasing budget. Numbers in the brackets represent values for $(K,r)$. The error is defined as $\mL_{\hat{i},\hat{j}}-\mL_{i_\star,j_\star}$ where $(\hat{i}, \hat{j})$ is a pair of optimal choices as estimated by each algorithm.. This is for the ML-100K dataset}}
\end{figure}